\documentclass[10pt, conference, letterpaper]{IEEEtran}

\IEEEoverridecommandlockouts

\usepackage[svgnames, x11names]{xcolor}

\usepackage{xspace}

\usepackage[caption=false]{subfig}
\usepackage{capt-of}

\usepackage{standalone}

\usepackage{enumitem}

\usepackage{blindtext}

\usepackage{longtable}
\usepackage{multirow} 
\usepackage{rotating} 
\usepackage{booktabs} 
\usepackage{colortbl} 
\usepackage{tablefootnote} 

\usepackage{array}
\definecolor{light-gray}{gray}{0.9} 
\newcolumntype{L}[1]{>{\raggedright\let\newline\\\arraybackslash\hspace{0pt}}m{#1}}
\newcolumntype{C}[1]{>{\centering\let\newline\\\arraybackslash\hspace{0pt}}m{#1}}
\newcolumntype{R}[1]{>{\raggedleft\let\newline\\\arraybackslash\hspace{0pt}}m{#1}}
\newcolumntype{G}[1]{>{\columncolor{light-gray}\centering\let\newline\\\arraybackslash\hspace{0pt}}m{#1}}

\usepackage{comment}

\usepackage{footnote}
\usepackage[cmex10]{amsmath}
\usepackage{amssymb}
\usepackage{mathtools}
\usepackage{amsthm}
\usepackage{amsfonts}

\usepackage{algorithmicx}
\usepackage{algpseudocode}
\usepackage[ruled, vlined, longend, linesnumbered]{algorithm2e}
\DontPrintSemicolon
\SetKwInOut{Input}{Input}
\SetKwInOut{Output}{Output}
\SetAlFnt{\small}
\newlength{\commentWidth}
\SetCommentSty{mycommfont}

\graphicspath{{figures/}}

\newtheorem{theorem}{Theorem}

\newcommand{\ysnote}[1]{ {\textcolor{magenta} { ***Yogesh: #1 }}} 

\newcommand{\fbsnote}[1]{ {\textcolor{blue} { ***Francesco: #1 }}}
\newcommand{\ysnoted}[1]{} 
\newcommand{\aknoted}[1]{}
\newcommand{\fbsnoted}[1]{}

\algnewcommand\Andalg{\xspace\upshape \textbf{and}\xspace}
\algnewcommand\Oralg{\xspace\upshape \textbf{or}\xspace}
\algnewcommand\Breakalg{\xspace\upshape \textbf{break}\xspace}
\algnewcommand\Falsealg{\xspace\upshape \textbf{false}\xspace}
\algnewcommand\Truealg{\xspace\upshape \textbf{true}\xspace}
\algnewcommand\Nullalg{\xspace\upshape \textbf{null}\xspace}

\newcommand*{\unit}[1]{\ensuremath{\mathrm{\,#1}}}

\newcommand{\depot}{\ensuremath{\widehat{\lambda}}\xspace}

\newcommand{\prob}{MSP\xspace} 

\newcommand{\algjsc}{\textsc{jsc}\xspace} 
\newcommand{\algvrc}{\textsc{vrc}\xspace} 
\newcommand{\algopt}{\textsc{opt}\xspace} 

\newcommand{\algjsclong}{\textsc{Job Scheduling Centric}\xspace}
\newcommand{\algvrclong}{\textsc{Vehicle Routing Centric}\xspace}
\newcommand{\algoptlong}{\textsc{Optimal Mission Scheduling}\xspace}
\newcommand{\algtwoopt}{\textsc{2-opt*}\xspace}
\newcommand{\algknn}{$k$\textsc{-nn}\xspace}

\newcommand{\algda}{\emph{default assignment}\xspace}
\newcommand{\algtas}{\emph{test and swap assignment}\xspace} 

\begin{document}
\title{Heuristic Algorithms for Co-scheduling of Edge Analytics and Routes for UAV Fleet Missions}

\author{
    \IEEEauthorblockN{
    	Aakash Khochare\IEEEauthorrefmark{1},
    	Yogesh Simmhan\IEEEauthorrefmark{1},
    	Francesco Betti Sorbelli\IEEEauthorrefmark{2}, and
    	Sajal K. Das\IEEEauthorrefmark{2}
    }
	\IEEEauthorblockA{
		\IEEEauthorrefmark{1}Department of Computer Science, Indian Institute of Science, India.
		Email: \{aakhochare, simmhan\}@iisc.ac.in\\
	}
    \IEEEauthorblockA{
    	\IEEEauthorrefmark{2}Dept. of Computer Science, Missouri Univ. of Science \& Tech., USA.
    	Email: \{francesco.bettisorbelli, sdas\}@mst.edu\\
    }
}	

\maketitle

\begin{abstract} 
Unmanned Aerial Vehicles (UAVs) or drones are increasingly used for urban applications like traffic monitoring and construction surveys. Autonomous navigation allows drones to visit \emph{waypoints} and accomplish \emph{activities} as part of their \emph{mission}. A common activity is to hover and observe a location using on-board cameras. Advances in Deep Neural Networks (DNNs) allow such videos to be analyzed for automated decision making. UAVs also host edge computing capability for on-board inferencing by such DNNs. To this end, for a fleet of drones, we propose a novel \emph{Mission Scheduling Problem (MSP)} that co-schedules the flight routes to visit and record video at waypoints, and their subsequent on-board edge analytics. The proposed schedule maximizes the utility from the activities while meeting activity deadlines as well as energy and computing constraints. We first prove that MSP is NP-hard and then optimally solve it by formulating a mixed integer linear programming (MILP) problem. Next, we design two efficient heuristic algorithms, \algjsc and \algvrc, that provide fast sub-optimal solutions. Evaluation of these three schedulers using real drone traces demonstrate utility--runtime trade-offs under diverse workloads.
\end{abstract} 

\begin{IEEEkeywords}
UAV, drone, edge computing, vehicle routing, job scheduling, energy constrained, video analytics, path planning
\end{IEEEkeywords}

\IEEEpeerreviewmaketitle

\section{Introduction}
Unmanned Aerial Vehicles (UAVs), also called \emph{drones}, are enabling a wide range of applications in smart cities~\cite{mohammed2014uavs}, such as traffic monitoring~\cite{kanistras2013survey}, construction surveys~\cite{george2019towards}, package delivery~\cite{sorbelli2020energy}, localization~\cite{sorbelli2020measurement}, and disaster (including COVID-19) management~\cite{costa2020covid}, assisted by 5G wireless roll-out~\cite{gapeyenko2018flexible}. 
The mobility, agility, and hovering capabilities of drones allow them to rapidly fly to points of interest (i.e., \emph{waypoints}) in the city to accomplish specific \emph{activities}. 
Usually, such activities involve hovering and recording a scene using the drone's camera, and analyzing the videos to take decisions. 

Advancements of computer vision algorithms and \emph{Deep Neural Networks} (DNNs) enable video analytics to be performed over such recordings for automated decision-making. 
Typically, these are inferred once the recordings are transferred to a ground station (GS) after the drones land. In-flight transfer of videos to a GS is limited by the intermittent bandwidth of current communications technologies.
However, certain activities may require low-latency analysis and decisions, as soon as the video is captured at a location. Hence, the on-board \emph{edge computing} capability~\cite{jung2018perception}
available on commercial drones can be leveraged to process the recorded videos, and quickly report concise results to the GS over 3/4/5G wireless networks~\cite{zeng2019accessing}.
Since the transferred results are brief and the \emph{on-board} processing times dominate, we ignore communication constraints like data rate, latency, and reliability that are affected by the UAV's altitude, antenna envelope, etc.


UAVs are \emph{energy-constrained vehicles} with limited battery capacity, and commercial drones can currently fly for less than an hour. The flying distance between waypoints will affect the number of activities that can be completed in one \emph{trip} on a full battery. Besides hovering and recording videos at waypoints, performing edge analytics also consumes energy. So, the drone's battery capacity should be judiciously managed for the flying, hovering and computing tasks. Nevertheless, once a drone lands, its exhausted battery can be quickly replaced with a full one, to be ready for a new trip.

This paper examines how a \emph{UAV fleet operator} in a city can plan \emph{missions} for a captive set of drones to accomplish activities periodically provided by the users. An \emph{activity} involves visiting a waypoint, hovering and capturing video at that location for a specific time period, and optionally performing on-board analytics on the captured data. Activities also offer \emph{utility} scores depending on how they are handled. The novel problem we propose here is for the fleet operator to \emph{co-schedule flight routing among waypoints \underline{and} on-board computation so that the drones complete (a subset of) the provided activities, within the energy and computation constraints of each drone, while maximizing the total utility.}

Existing works have examined routing of one or more drones for 
capturing and relaying data to the backend~\cite{motlagh2019energy}, off-loading computations from mobile devices~\cite{hu2019uav}, and cooperative video surveillance~\cite{trotta2018uavs}.
There also exists literature on scheduling tasks for edge computing that are compute- and energy-aware, operate on distributed edge resources, and consider deadlines and device reliability~\cite{meng2019dedas}. However, none of these examine co-scheduling a fleet of physical drones and digital applications on them to meet the objective, while efficiently managing the energy capacity to maximize utility.

Specifically, our \emph{Mission Scheduling Problem (MSP)} combines elements of the \emph{Vehicle Routing Problem (VRP)}~\cite{clarke1964scheduling}, which generalizes the well known Traveling Salesman Problem (TSP) to find optimal routes for a set of vehicles and customers~\cite{toth2002vehicle}, and the \emph{Job-shop Scheduling Problem (JSP)}~\cite{manne1960job} for mapping jobs of different execution duration to the available resources, which is often used for parallel scheduling of computing tasks to multiprocessors~\cite{kwok1999static}.

\ysnoted{Experiment results...anything that stands out? Counter-intuitive?}

We make the following contributions in this paper.
\begin{itemize}
    \item We characterize the system and application model, and formally define the \textit{Mission Scheduling Problem (\prob)} to co-schedule routes and analytics for a fleet of drones, maximizing the obtained utility (Sections~\ref{sec:model} and~\ref{sec:prob-def}).
    \item We prove that \prob is \textit{NP-hard}, and optimally solve it using a \textit{mixed integer linear programming (MILP)} design, \algopt, which is feasible for small inputs (Section~\ref{sec:algorithms:opt}).
    \item We design \textit{two time-efficient heuristic algorithms, \algjsc and \algvrc,} that solve the MSP for arbitrary-sized inputs, and offer complexity bounds for their execution (Section~\ref{sec:algorithms:appox}).
    \item We \textit{evaluate and analyze} the utility and scheduling runtime trade-offs for these three algorithms, for diverse drone workloads based on real drone traces (Section~\ref{sec:evaluation}).
\end{itemize}


\section{Related Work}\label{sec:related}
This section reviews literature on vehicle routing and job-shop scheduling, contrasting them with MSP and our solutions.

\subsection{Vehicle Routing Problem (VRP)}
VRP is a variant of TSP with multiple salespersons~\cite{clarke1964scheduling} and it is NP-hard~\cite{lenstra1981complexity}.
This problem has had several extensions to handle realistic scenarios, such as temporal constraints that impose deliveries only at specific time-windows~\cite{desaulniers2016exact}, capacity constraints on vehicle payloads~\cite{uchoa2017new}, multiple trips for vehicles~\cite{cattaruzza2016vehicle}, profit per vehicle~\cite{stavropoulou2019vehicle} and traffic congestion~\cite{gayialis2018developing}.
VRP has also been adapted for route planning for a fleet of ships~\cite{fagerholt1999optimal}, and for drone assisted delivery of goods~\cite{khoufi2019survey}. 

In~\cite{motlagh2019energy} 
the scheduling of \emph{events} is performed by UAVs at specific locations, involving data sensing/processing and communication with the GS. The goal here is to minimize the drone's energy consumption and operation time. Factors like wind and temperature that may affect the route and execution time are also considered. 
While they combine sensing and processing into one monolithic event, these are independent tasks which need to be co-scheduled, as we do. Also, they minimize the operating time and energy while we maximize the utility to perform tasks within a time and energy budget.

In~\cite{hu2019uav} the use of UAVs is explored to off-load computing from the users' mobile devices, and for relaying data between mobile devices and GS. The authors considered the drones' trajectory, bandwidth, and computing optimizations in an iterative manner. The aim is to minimize energy consumption of the drones and mobile devices. It is validate through simulation for four mobile devices. We instead consider a more practical problem for a fleet of drones with possibly hundreds of locations to visit and on-board computing tasks to perform.

Trotta et al.~\cite{trotta2018uavs} propose a novel architecture for energy-efficient video surveillance of points of interest (POIs) in a city by drones. The UAVs use bus rooftops for re-charging and being transported to the next POI based on known bus routes. Drones also act as relays for other drones capturing videos. The mapping of drones to bus routes is formulated as an MILP problem and a TSP-based heuristic is proposed. Unlike ours, their goal is not to schedule and process data on-board the drone. Similarly, we do not examine any data off-loading from the drone, nor any piggy-backing mechanisms.


\subsection{Job-shop Scheduling (JSC)}
Scheduling of computing tasks on drones is closely aligned with scheduling tasks on edge and fog devices~\cite{varshney2020characterizing}, and broadly with parallel workload scheduling~\cite{kwok1999static} and JSC~\cite{manne1960job}.

In~\cite{meng2019dedas}, an online algorithm is proposed for deadline-aware task scheduling for edge computing. It highlights that workload scheduling on the edge has several dimensions, and it jointly optimizes networking and computing to yield the best possible schedule. 
Feng at al.~\cite{feng2018mobile} propose a framework for cooperative edge computing on autonomous road vehicles, which aims to increase their decentralized computational capabilities and task execution performance.
Others~\cite{li2019joint} combine optimal placement of data blocks with optimal task scheduling to reduce computation delay and response time for the submitted tasks while improving user experience in edge computing.
In contrast, we co-schedule UAV routing and edge computing.

There exist works that explore task scheduling for mobile clients, and off-load computing to another nearby edge or fog resource. These may be categorized based on their use of \emph{predictable} or \emph{unpredictable} mobility models. In~\cite{ning2019mobile}, the mobility of a vehicle is predicted and used to select the road-side edge computing unit to which the computation is off-loaded. Serendipity~\cite{shi2012serendipity} takes an alternate view and assumes that mobile edge devices interacts with each other intermittently and at random. This makes it challenging to determine if tasks should be off-loaded to another proximate device for reliable completion. The problem we solve is complementary and does not involve off-loading. The possible waypoints are known ahead, and we perform predictable UAV route planning and scheduling of the computing locally on the edge.

Scheduling on the energy-constrained edge has also been investigated by Zhang et al.~\cite{zhang2017energy}, where an energy-aware off-loading scheme is proposed to jointly optimize communication and computation resource allocation on the edge, and to limit latency. Our proposed problem also considers energy for the drone flight while meet deadlines for on-board computing.
\section{Models and Assumptions}\label{sec:model}
This section introduces the UAV system model, application model, and utility model along with underlying assumptions. 

\begin{figure*}[t]
	\centering
	\def\svgscale{1}
	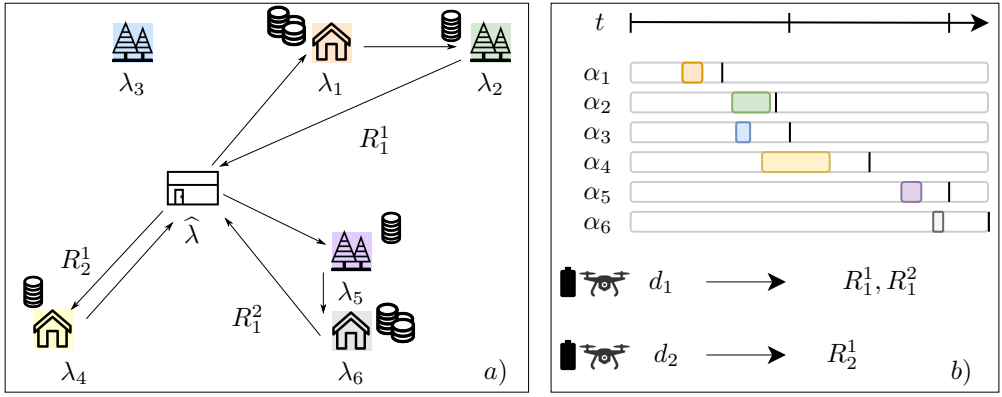
	\caption{Sample MSP scenario. a) shows a city with the depot ($\depot$); 6 waypoints to visit ($\lambda_i$) with some utility; and possible trip routes for drones ($R^i_j$). b) has the corresponding 6 activities ($\alpha_i$) with data capture duration (shaded) and compute deadline (vertical line) and the two available drones.}
	\label{fig:msp_big_picture}
\end{figure*}

\subsection{UAV System Model}
Let $\depot=(0, 0, 0)$ be the \emph{location} of a UAV depot in the city (see Figure~\ref{fig:msp_big_picture}, left) centered at the origin of a 3D Cartesian coordinate system. 
Let $D = \{d_1, \ldots, d_m\}$ be the set of $m$ available drones. For simplicity, we assume that all the drones are homogeneous. Each drone has a camera for recording videos, which is subsequently processed. This processing can be done using the on-board computing, or done offline once the drone lands (which is outside the scope of our problem). The on-board \emph{processing speed} is $\pi$ floating point operations per second (FLOPS). For simplicity, this is taken as cumulative across CPUs and GPUs on the drone, and this capacity is orthogonal to any computation done for navigation.

The battery on a drone has a fixed \emph{energy capacity} $E$, which is used both for flying and for on-board computation.
The drone's energy consumption has three components --
\emph{flying}, \emph{hovering} and \emph{computing}.
Let $\epsilon^f$ be the energy required for flying for a unit time duration at a constant energy-efficient speed $s$ within the Cartesian space; 
let $\epsilon^h$ be the energy for hovering for a unit time duration;
and let $\epsilon^c$ be the energy for performing computation for a unit time duration.
For simplicity, we ignore the energy for video capture since it is negligible in practice.
Also, a drone that returns to the depot can swap-in a full battery and immediately start a new trip.
\subsection{Application Model}
Let $A = ( \alpha_1, \ldots, \alpha_n )$ be the set of $n$ activities to be performed starting from time $\widehat{t}=0$, where each {\em activity} $\alpha_i$ is given by the tuple $\langle \lambda_i, t_i, \bar{t}_i, \kappa_i, \delta_i, \gamma_i, \bar{\gamma}_i, \bar{\bar{\gamma}}_i  \rangle$. Here, $\lambda_i=(x_i, y_i, z_i)$ is the waypoint \emph{location} coordinates where the video data for that activity has to be captured by the drone, relative to the depot location $\depot$. The \emph{starting and ending times} for performing the \emph{data capture task} are $t_i$ and $\bar{t}_i$. The \emph{compute requirements} for subsequently processing all of the captured data is $\kappa_i$  floating point operations. 
Lastly, $\delta_i$ is the \emph{time deadline} by which the \emph{computation task} should be completed on the drone to derive on-time utility of processing, while $\gamma_i, \bar{\gamma}_i$, and $\bar{\bar{\gamma}}_i$ are the \emph{data capture}, \emph{on-time processing} and \emph{on-board processing utility} values that are gained for completing the activity. These are described in the next sub-section.

The computation may be performed incrementally on subsets of the video data, as soon as they are captured. This is common for analytics over constrained resources~\cite{bianco2018benchmark}. 
Specifically, for an activity $\alpha_{i}$, the data captured between $(\bar{t}_{i} - t_{i})$ is divided into \emph{batches} of a fixed duration $\beta$, with the sequence of batches given by $B_{i} = (b_{i}^1, \ldots, b_{i}^{q_i})$, where $q_i = |B_i| = \big\lceil \frac{\bar{t}_{i} - t_{i}}{\beta} \big\rceil$. The computational cost to process each batch is $\kappa_{i}^k = \frac{\kappa_{i}}{q_i}$ floating-point operations, and is constant for all batches of an activity.
So, the \emph{processing time} for the batch, given the processing speed $\pi$ for a drone, is  $\rho_{i}^k = \big\lceil\kappa_{i}^k \cdot \frac{1}{\pi}\big\rceil$; for simplicity, we discretize all time-units into integers.

We make some simplifying assumptions. Only one batch may be executed at a time on-board a drone and it should run to completion before scheduling another. There is no concurrency, pre-emption or check-pointing. The data capture for an activity's batch may overlap with the computation of a previous batch of the same or a different activity. 
All batches for a single activity should be executed in sequence, i.e., complete processing $b_{i}^{k}$ before processing $b_{i}^{k+1}$. Once a batch is processed, its compact results are immediately and deterministically communicated to the GS.


\subsection{Utility Model}

The primary goal of the drone is to capture videos at the various activity locations for the specified duration. This is a \emph{necessary} condition for an activity to be successful. We define this as the \emph{data capture utility ($\gamma_{i}$)} accrued by a drone for an activity $\alpha_{i}$.
The secondary goal is to opportunistically process the captured data using the on-board computing on the drone. Here, we have two scenarios. Some activities may not be time sensitive, and performing on-board computing is just to reduce the costs for offline computing. Here, processing the data captured by an activity using the drone's computing resources will provide an \emph{on-board processing utility ($\bar{\gamma}_{i}$)}. Other activities may be time-sensitive and have a \emph{soft-deadline} $\delta_i$ for completing the processing. For these, if we process its captured data on the drone by this deadline, we receive an extra \emph{on-time processing utility ($\bar{\bar{\gamma}}_{i}$)}. The processing utilities accrue \emph{pro rata}, for each batch of the activity completed.


\section{Problem Formulation}\label{sec:prob-def}
The Mission Scheduling Problem (\prob) is summarized as: \emph{Given a UAV depot in a city with a fleet of captive drones, and a set of observation and computing activities to be performed at locations in the city, each within a given time window and with associated utilities, the goal is to co-schedule the drones onto mission routes and the computation to the drones, within the energy and compute constraints of the drones, such that the total utility achieved is maximized.}
It is formalized below.

\subsection{Mission Scheduling Problem (\prob)}
A UAV fleet operator receives and queues activities. Periodically, a mission schedule is planned to serve some or all the activities using the whole fleet to maximize the utility. There is a fixed cost for operating the captive fleet that we ignore.

Multiple activities can be assigned to the same drone $d_j$ as part of the drone's \emph{mission},
and the same drone $d_j$ can perform multiple \emph{trips} from the depot for a mission.
The \emph{mission activities} for the $r^{th}$ trip of a drone $d_j$ is the ordered sequence $A^r_j = ( \alpha^r_{j_1}, \ldots, \alpha^r_{j_n} ) \subseteq A$ where $\alpha^r_{j_x} \in A$, $j_n \le n$, and no activity appears twice within a mission. Further, we have $\alpha^r_{j_x} \prec \alpha^r_{j_{x+1}}$, i.e., the observation start and end times of an activity in the mission sequence fully precede those of the next activity in it, $\bar{t}^r_{j_{x}} \le t^r_{j_{x+1}} $. Also, $A^x_j \cap A^y_k = \varnothing~\forall j,k,x,y$ to ensure that an activity is mapped to just one drone. Depending on the feasibility and utility, some activities may not be part of any mission and are dropped, i.e., $\sum_{j} \sum_{r} |A^r_j| \leq n$.

The \emph{route} for the $r^{th}$ trip of drone $d_j$ is given by $R^r_j = ( \depot, \lambda^r_{j_1}, \ldots, \lambda^r_{j_n}, \depot )$, where the starting and ending waypoints of the drone are the depot location $\depot$, and each intermediate location corresponds to the video capture location $\lambda^r_{j_k}$ for the activity $\alpha^r_{j_k}$ in the mission sequence. For uniformity, we denote the first and the last depot location in the route as $\lambda^r_{j_0}$ and $\lambda^r_{j_{n+1}}$, respectively.
Clearly, $|R^r_j| = j_n + 2$.


A drone $d_j$, given the $r^{th}$ trip of its route $R^r_j$, starts at the depot, visits each waypoint in the sequence and returns to the depot, where it may instantly get a fresh battery and start the $(r+1)^{th}$ route. 
Let drone $d_j$ leave a waypoint location in its route, $\lambda^r_{j_i}$, at \emph{departure time} $\tau^r_{j_i}$ and reach the next waypoint location, $\lambda^r_{j_{i+1}}$, at \emph{arrival time} $\bar{\tau}^r_{j_{i+1}}$.
Let the function $\mathcal{F}(\lambda_p, \lambda_q)$ give the \emph{flying time} between $\lambda_i$ and $\lambda_j$. Since the drone has a constant flying speed, we have $\bar{\tau}^r_{j_{i+1}} = \tau^r_{j_i} + \mathcal{F}(\lambda^r_{j_i}, \lambda^r_{j_{i+1}})$.

The drone must hover at each waypoint $\lambda^r_{j_i}$ between $t^r_{j}$ and $\bar{t}^r_{j}$ while recording the video, and it departs the waypoint after this, i.e., $\tau^r_{j_i} = \bar{t}^r_{j_i}$. If the drone arrives at this waypoint at time $\bar{\tau}^r_{j_i}$, i.e., before the observation start time $t_j$, it \emph{hovers} here for a duration of $t^r_j - \bar{\tau}^r_{j_i}$, and then continues hovering during the activity's video capture. 
If a drone arrives at $\lambda^r_{j_i}$ after $t^r_j$, it is invalid since the video capture for the activity cannot be conducted for the whole duration. So, $\bar{\tau}^r_{j_i} \le t^r_{j_i} \le \tau^r_{j_i}$.
Also, since the deadline for on-time computation over the captured data is $\delta^r_{j_i}$, we require $\delta^r_{j_i} \ge \bar{t}^r_{j_i}$. 
Once the drone finishes capturing video for the last activity in its $r^{th}$ trip, it returns back to the depot location at time $\bar{\tau}^r_{j_{n+1}} = \tau^r_{j_n} + \mathcal{F}(\lambda^r_{j_n}, \widehat{\lambda})$.
Hence, the \emph{total flying time} for a drone $d_j$ for its $r^{th}$ trip is:
\[ f^r_j = \sum_{i=0}^{n} (\bar{\tau}^r_{j_{i+1}} - \tau^r_{j_i}) \]
%

\noindent
and the \emph{total hover time} for the drone on that trip is:

\[ h^r_j = \sum_{i=1}^{n} (t^r_{j_i} - \bar{\tau}^r_{j_i}) + \sum_{i=1}^{n} (\bar{t}^r_{j_i} - t^r_{j_i}) = \sum_{i=1}^{n} (\bar{t}^r_{j_i} - \bar{\tau}^r_{j_i}) \]
which includes hovering due to early arrival at a waypoint, and hovering during the data capture.


Let the scheduler assign the \emph{time slot} $[\theta_{j_i}^{k}, \bar{\theta}_{j_i}^{k})$ for executing a batch $b_{j_i}^{k}$ of activity $\alpha_{j_i}$ on drone $d_j$, where $\bar{\theta}_{j_i}^{k} = \theta_{j_i}^{k} + \rho_{i}^{k}$, based on the batch execution time.
%
%
%
%
We define a completion function for each activity $\alpha_{j_i}$, for the three utility values:

\begin{itemize}
	\item The \emph{data capture completion} $u_{j_i} \in \{0, 1\}$ has a value of $1$ if the drone hovers at location $\lambda_{j_i}$ for the entire period from $t_{j_i}$ to $\bar{t}_{j_i}$, and is $0$ otherwise. 
	
    \item The \emph{on-board completion} $0.0 \le \bar{u}_{j_i} \le 1.0$ indicates the fraction of batches of that activity that are completed on-board the drone. 
    Let $\bar{\mu}^k_i=1$ if the batch $b^k_i$ of activity $\alpha_i$ is completed on-board, 
    and $\bar{\mu}^k_i=0$ if it is not completed on-board the drone. Then, $\bar{u}_{j_i} = \frac{\sum_k \bar{\mu}^k_i}{q_i}$.
    
    \item The \emph{on-time completion} $0.0 \le \bar{\bar{u}}_{j_i} \le 1.0$ gives the fraction of batches of that activity that are fully completed within the deadline. 
    As before, let $\bar{\bar{\mu}}^k_i=1$ if the batch $b^k_i$ of activity $\alpha_i$ is completed 
    on-time, i.e., $\bar{\theta}_{i}^{k} \leq \delta_i$, and $\bar{\bar{\mu}}^k_i=0$ otherwise. So, $\bar{\bar{u}}_{j_i} = \frac{\sum_k \bar{\bar{\mu}}^k_i}{q_i}$.
\end{itemize}

The \emph{total utility} for an activity $\alpha_i$ is
$U_i = u_{i} \gamma_i  + \bar{u}_{i} \bar{\gamma}_i + \bar{\bar{u}}_{i} \bar{\bar{\gamma}}_i$,
and the \emph{total computation time} of batches on a drone $d_j$ is:
\[c_j = \sum_{\alpha_i \in A} {(\bar{\mu}^k_{j_i} + \bar{\bar{\mu}}^k_{j_i}) \cdot \rho^k_i} \]

\subsection{Optimization of \prob}
Based on these, the \textbf{objective} of the optimization is $\arg \max \sum_{\alpha_i \in A} {U_i}$, i.e., assign drones to activity waypoints and activity batches to the drones' computing slots to maximize the utility from data capture, on-board and on-time computation.
These are subject to the following constraints on the execution slot assignments for a batch on a drone:
\[ (t_{j_i} + k \cdot \beta) \leq \theta_{j_i}^{k} ~\qquad~ \bar{\theta}_{j_i}^{k} \leq \theta_{j_i}^{k+1} ~\qquad~ \bar{\theta}_{i}^{k} \leq \bar{\tau}_{j_{n+1}}\]
i.e., the data capture for a duration of $\beta$ for the $k^{th}$ batch of the activity is completed before the execution slot of the batch starts; the batches for an activity are executed in sequence; and the execution completes before the drone lands.

Also, there can only be one batch executing at a time on a drone. So  $\forall [\theta_{j_p}^{x}, \bar{\theta}_{j_p}^{x})$ and $[\theta_{j_q}^{y}, \bar{\theta}_{j_q}^{y})$ slots assigned to batches $b_p^x$ and $b_q^y$ on drone $d_j$, we have $[\theta_{j_p}^{x}, \bar{\theta}_{j_p}^{x}) \cap [\theta_{j_q}^{y}, \bar{\theta}_{j_q}^{y}) = \varnothing$.

Lastly, the \emph{energy expended} by drone $d_j$ on the $r^{th}$ trip, to fly, hover and compute, should be within its battery capacity:

\[
    E^r_j =  f^r_j \epsilon^f + h^r_j \epsilon^h + c^r_j \epsilon^c \le E
\]

\noindent{\em Model Applicability:}
Our novel model can be abstracted to describe diverse applications.
In \emph{entity localization}~\cite{de2015board}, $\bar{\gamma}_{i}=0$ and $\bar{\bar{\gamma}}_{i}>0$ captures the importance of an entity being tracked.
In \emph{traffic monitoring}~\cite{kanistras2013survey} it is useful to have timely insights, appropriately tuning $\bar{\gamma}_{i}$ and $\bar{\bar{\gamma}}_{i}$. 
In \emph{construction survey}~\cite{george2019towards} there are no strict time deadlines, so $\bar{\bar{\gamma}}_{i}=0$.


\begin{table*}[t]
\centering
\caption{Constraints for \algopt MILP formulation.}
\label{tab:constraints}
\begin{tabular}{p{0.15cm}|p{9.5cm}|p{7cm}}
\toprule
\bf C. & \bf Expression & \bf Meaning \\
\midrule

$1$ & $\sum_{k \in \mathcal{D}} \sum_{l \in \mathcal{R}} \sum_{j\in \overrightarrow{i}} x_{ij}^{kl} \leq 1, 
\qquad \forall i \in  \mathcal{V}'$  & The waypoint for an activity $\alpha_i$ is visited only once.\\

$2$ & $\sum_{j \in \overrightarrow{0}} x_{0j}^{kl} - \sum_{j \in \overleftarrow{0}} x_{j0}^{kl} = 0, 
\qquad \forall k \in \mathcal{D}, l \in \mathcal{R}$  &  A drone trip $l$ starting from the depot must also end there.\\

$3$ & $\sum_{j \in \overrightarrow{0}} x_{0j}^{kl} = 1 \iff \sum_{j \in \overrightarrow{i}} x_{ij}^{kl} = 1, 
\qquad \forall i \in  \mathcal{V}', k \in \mathcal{D}, l \in \mathcal{R}$  &  A drone $k$ must visit at least one waypoint on each trip $l$.\\

$4$ & $\sum_{i \in \overleftarrow{j}} x_{ij}^{kl} - \sum_{i \in \overrightarrow{j}} x_{ji}^{kl} = 0,  
\qquad \forall k \in \mathcal{D}, j \in \mathcal{V}', l \in \mathcal{R}$  &  A drone $k$ visiting waypoint $j$ must also fly out from there.\\

$5$ & $\left(t_j - \mathcal{F}_{0j}\right) \cdot \sum_{k \in \mathcal{D}}\sum_{l \in \mathcal{R}} x_{0j}^{kl} \ge 0, 
\qquad \forall j \in \mathcal{V}'$  &  Any drone flying to waypoint $j$ from the depot must reach before its observation start time $t_j$.\\

$6$ & $ (t_j - \bar{t}_i - \mathcal{F}_{ij}) \cdot \sum_{k \in \mathcal{D}}\sum_{l \in \mathcal{R}} x_{ij}^{kl} \ge 0, 
\qquad \forall i \in \mathcal{V}', j \in \overrightarrow{i}$  &  Any drone flying to waypoint $j$ from $i$ must reach before its observation start time $t_j$.\\

${7}$ & $\bar{\tau}^l_{k_{n+1}} = \sum_{i \in \mathcal{V}'} x_{i0}^{kl} \cdot (\bar{t}_i + \mathcal{F}_{i0}),
\qquad \forall k \in \mathcal{D}, l \in \mathcal{R}$  &  Decides the landing time of drone $k$ at the depot after trip $l$.\\

${8}$ & $\bar{\tau}^l_{k_{n+1}} \le \tau_{\max}, 
\quad \forall k \in \mathcal{D}, l \in \mathcal{R}$  &  Depot landing times for all trips is within the maximum time.\\

\hline

$9$ & $t_i + (g + 1) \cdot \beta \le \theta^g_i,
\qquad \forall i \in \mathcal{V}', g \in \mathcal{B}_i$  &  Batch $g$ of activity $\alpha_i$ must be observed before it is processed. \\

${10}$ & $\bar{\theta}^g_i < \theta^{g+1}_i,
\qquad \forall i \in \mathcal{V}', g \in \mathcal{B}_i$  &  Processing of batch $g$ of activity $\alpha_i$ must precede batch $g+1$. \\

${11}$ & $\sum_{j \in \overrightarrow{i}} x_{ij}^{kl} + \sum_{b \in \overrightarrow{a}} x_{ab}^{kl} - 1 \leq w^{gh}_{ia} + w^{hg}_{ai},
\qquad \forall i,a \in \mathcal{V}', i < a, g \in \mathcal{B}_i, h \in \mathcal{B}_a, k \in \mathcal{D}, l \in \mathcal{R}$   &  \multirow{2}{7cm}{Compute time slots of two batches $g$ and $h$ from activities $\alpha_i$ and $\alpha_a$ on the same drone $k$ and trip $l$ should not overlap \cite{manne1960job}.} \\

${12}$ & $\bar{\theta}^g_i - \theta^h_a \leq M \cdot (1 - w^{gh}_{ia}),
\qquad \forall i,a \in \mathcal{V}, i \neq a, g \in \mathcal{B}_i, h \in \mathcal{B}_a $  &  \\

${13}$ & $y_{ik}^{lg} = 1 \Rightarrow \bar{\theta}^g_i + M \left(1-\sum_{j \in \overrightarrow{i}} x_{ij}^{kl}\right) \le \delta_i, ~~~~~\forall i \in \mathcal{V}', g \in \mathcal{B}_i, k \in \mathcal{D}, l \in \mathcal{R}$  &  Decision variable for batches that complete before deadline.\\

${14}$ & $z_{ik}^{lg} = 1 \Rightarrow \bar{\theta}^g_i + M\left(1 - \sum_{j \in \overrightarrow{i}} x_{ij}^{kl}\right) \le \bar{\tau}^l_{k_{n+1}},
\forall i \in \mathcal{V}', g \in \mathcal{B}_i, k \in \mathcal{D}, l \in \mathcal{R}$  &  Decision variable for batches that complete before landing.\\

\hline

${15}$ & $\sum_{i \in \mathcal{V}} \Big( \sum_{j \in \overrightarrow{i}} \left( x_{ij}^{kl} \cdot \mathcal{F}_{ij} \cdot \epsilon^f \right) ~+~  \allowbreak
\sum_{g \in \mathcal{B}_i} \left(z_{ik}^{lg} \cdot \kappa^g_i \cdot \epsilon^c \right) ~+~ \allowbreak
{\qquad \sum_{j \in \overrightarrow{i}} \left( x_{ij}^{kl} \cdot (\bar{t}_j \!-\! (\bar{t}_i + \mathcal{F}_{ij})) \cdot \epsilon^h \right)} \Big)  
\le E,
\quad \forall k \in \mathcal{D}, l \in \mathcal{R}$  &  Sum of energy consumed for flying, hovering and computing on trip $l$ of drone $k$ should be within the battery capacity. \\





\bottomrule
\end{tabular}
\end{table*}

\section{Optimal Solution for \prob}
\label{sec:algorithms:opt}

In this section, we prove that \prob is NP-hard, and we define an optimal, but computationally slow, algorithm called \algoptlong (\algopt) based on MILP.
\ysnoted{We're actually defining a solution here, not a problem. Should we call it "Optimal Mission Scheduler (OPT)"? The short form OPT is also earlier to remember for readers.}

\subsection{NP-hardness of \prob}
As discussed earlier, the \prob combines elements of the VRP and the JSP in assigning routes and batches to drones, for maximizing the overall utility, subject to energy constraints.

\begin{theorem}
\prob is NP-hard.
\end{theorem}
\begin{proof}
The VRP is NP-hard~\cite{lenstra1981complexity}.
In addition, \prob considers multiple-trips, time-windows, energy-constraints, and utilities.

The VRP variant with multiple-trips (MTVRP), which considers a maximum travel time horizon $T_{h}$, is NP-hard. Any instance of VRP can be reduced in polynomial time to MTVRP by fixing the number of vehicles to the number of waypoints, $m = n$, and setting the time horizon $T_{h} = \sum_{e \in \mathcal{E}} \mathcal{F}(e)$, where $\mathcal{E}$ is the set of edges and $\mathcal{F}(e)$ is the flying time for traversing an edge~\cite{olivera2007adaptive}, and limiting the number of trips to one.
The VRP variant with time-windows (TWVRP), which limits the start and end time for visiting a vertex, $[t_i, \bar{t}_i)$, is NP-hard. Any instance of VRP can be reduced in polynomial time to TWVRP by just setting $t_i = 0$ and $\bar{t}_i = + \infty$~\cite{toth2002vehicle}.
Clearly, a VRP variant with energy-constrained vehicles is still NP-hard, by just relaxing those constraints to match VRP.

In the above VRP variants, the goal is only to minimize the costs. But \prob aims at maximizing the utility while bounding the energy and compute budget.
In literature, the VRP variant with profits (PVRP) is NP-hard~\cite{cattaruzza2016vehicle} since any instance of MTVRP can be reduced in polynomial time to PVRP by just setting all vertices to have the same unit-profit.
Moreover, \prob has to deal with scheduling of batches for maximizing the profit.
The original JSP is NP-hard~\cite{graham1966bounds}. So, any variant which introduces constraints is again NP-hard by a simple reduction, by relaxing those constraints, to JSP.

As \prob is a variant of VRP and JSP, it is NP-hard too.
\end{proof}

\subsection{The \algopt Algorithm}
The \algoptlong (\algopt) algorithm offers an optimal solution to \prob 
by modeling it as a multi-commodity flow problem (MCF), similar to~\cite{zmazek2006multiple, trotta2018uavs}. 
We reformulate the \prob definition as an MILP formulation.

The paths in the city are modeled as a \emph{complete graph}, $\mathcal{G} = (\mathcal{V}, \mathcal{E})$, between the $n$ activity waypoint vertices, $\mathcal{V} = \{0, 1, \ldots, n\}$, where $0$ is the depot \depot.
Let $\overrightarrow{i}$ and $\overleftarrow{i}$ be the set of \emph{out-edges} and \emph{in-edges} of a vertex $i$, and  $\mathcal{V}' = \mathcal{V} \setminus \{ 0 \}$ be the set of all waypoint vertices.
We enumerate the $m$ drones as $\mathcal{D} = \{1, \ldots, m\}$. 
Let $\tau_{\max}$ be the maximum time for completing all the missions, and
$r_{\max}$ the maximum trips a drone can do. Let $\mathcal{R} = \{1, \ldots, r_{\max}\}$ be the possible trips.

Let $x_{ij}^{kl} \in \{0,1\}$ be a decision variable that equals $1$ if the drone $k \in \mathcal{D}$ in its trip $l \in \mathcal{R}$ traverses the edge $(i,j)$, and $0$ otherwise. If $x_{ij}^{kl} = 1$ for $i \in \mathcal{V}'$, then the waypoint for activity $\alpha_i$ was visited by drone $k$ on trip $l$.
Let $\mathcal{B}_i = \{0, \ldots, q_i\}$ be the set of batches of activity $\alpha_i$.
Let $w^{gh}_{ia}$ be a binary decision variable used to linearize the batch computation  whose value is $1$ if batch $b^g_i$ is processed before $b^h_a$, and $0$ otherwise~\cite{manne1960job}.

Let $y_{ig}^{kl}$ be a decision variable that equals $1$ if the drone $k \in \mathcal{D}$ in trip $l \in \mathcal{R}$ processes the batch $g$ of activity $\alpha_i$ within its deadline $\delta_i$, and $0$ otherwise; and similarly, $z_{ig}^{kl}$ equals $1$ if the batch is processed before the drone completes the trip and lands, and $0$ otherwise.
Let the per batch utility for on-board completion be $\bar{\Gamma}_i = \frac{\bar{\gamma}_i}{q_i}$, and on-time completion be $\bar{\bar{\Gamma}}_i = \frac{\bar{\bar{\gamma}}_i}{q_i}$, for activity $\alpha_i$.
Finally, let $M$ be a sufficiently large constant.


Using these, the MILP objective is:

{\small 
\begin{align}
\max \sum_{k \in \mathcal{D}} \sum_{l \in \mathcal{R}} \sum_{i \in \mathcal{V}} \bigg( \sum_{j \in \overrightarrow{i}}  x_{ij}^{kl} \cdot \Gamma_i \bigg)\!\! + \!\!\bigg( \sum_{g \in \mathcal{B}_i} y_{ig}^{kl} \cdot \bar{\Gamma_i}  +  z_{ig}^{kl} \cdot \bar{\bar{\gamma_i}}\bigg) \label{eq:main}
\end{align}
}

\noindent subject to the constraints listed in Table~\ref{tab:constraints}.



\section{Heuristic Algorithms for \prob}
\label{sec:algorithms:appox}

Since \prob is NP-hard, \algopt is tractable only for small-sized inputs.
So, time-efficient but sub-optimal algorithms are necessary for larger-sized inputs.
In this section, we propose two heuristic algorithms, called \algjsclong (\algjsc) and \algvrclong (\algvrc).


\subsection{The \algjsc Algorithm}
The \algjsclong (\algjsc) algorithm aims to find near-optimal scheduling of batches while ignoring the optimizations of routing to conserve energy.
\algjsc is split into two phases: \emph{clustering} and \emph{scheduling}.

\subsubsection{Clustering Phase}
First, we use the ST-DBSCAN algorithm~\cite{birant2007st} to find time-efficient spatio-temporal clusters of activities.
It returns a set of clusters $\mathbb{C}$ such that for activities within a cluster $C_i \in \mathbb{C}$, certain spatial and temporal distance thresholds are met.
Drones are then allocated to clusters depending on their availability.
For each cluster $C_i$, let $T_i^U = \max_{\alpha_j \in C_i}{(\bar{t}_j + \mathcal{F}(\lambda_j, \depot))}$ be the upper bound for the \emph{latest landing time} for a drone servicing activities in $C_i$;
analogously, let $T_i^L = \min_{\alpha_j \in C_i}{(t_j-\mathcal{F}(\depot, \lambda_j))}$  
be the lower bound for the \emph{earliest take-off time}.
Then, all the temporal windows $[T_i^L, T_i^U]$ for each $C_i \in \mathbb{C}$ are sorted with respect to $T_i^L$.
Recalling that there are $m$ drones available at $\widehat{t}=0$, they are proportionally allocated to clusters depending on the current availability, which in turn depends on the temporal window.
So, $c_1 = \frac{m}{n}\cdot |C_1|$ drones are allocated to $C_1$ at time $T_1^L$ and released at time $T_1^U$; $c_2 = \frac{m-c_1}{n} \cdot |C_2|$ allocated to $C_2$ from $T_2^L$ to $T_2^U$ (assuming $T_2^L < T_1^U$), and so on.
\ysnoted{Just to be clear, in this e.g., the time window of C1 and C2 overlap and hence we have (m-c1)? Else it would be m. \fbsnote{yes yes, they overlap.}}

\subsubsection{Scheduling Phase}
Here, the activities are assigned to drones.
The \emph{feasibility} of assigning $\alpha_{i}$ to $d_j$, is tested by checking if the required flying and hovering energy is enough to visit $A_j \cup \alpha_i$; 
here, we ignore the batch processing energy. 
If feasible, the drone can update its take-off and landing times accordingly, and then schedule the subset of batches $\widehat{B_i} \subseteq B_i$ within the energy requirements.
Assignments are done in two steps: \algda and \algtas.

\noindent {\textbf{Default Assignment.}}
For each $b_i^k \in \widehat{B_i}$, let $P_{b^k_i} = [t_k + i \beta, \delta_k)$ be the \emph{preferred interval}; $Q_{b^k_i} \subseteq P_{b^k_i}$ be the \emph{available preferred sub-intervals}, i.e., the set of periods where no other batch is scheduled; and $S_{b^k_i} = [\delta_k, \bar{\tau}_{j_{n+1}})$ be the \emph{schedulable interval}, which exceeds the deadline but completes on-board.\ysnoted{Swapping P and Q; P is now "preferred"}\fbsnoted{Great idea, P as preferred!}
Clearly, $P_{b^k_i} \cap S_{b^k_i} = \varnothing$.
The \emph{default schedule} determines a suitable time-slot for $b_i^k$.
If $Q_{b^k_i} \neq \varnothing$, $b_i^k$ is \emph{first-fit} scheduled within intervals of $Q_{b^k_i}$; else, if $Q_{b^k_i} = \varnothing$, the same first-fit policy is applied over intervals of $S_{b^k_i}$.
If $b_i^k$ cannot be scheduled even in $S_{b^k_i}$, it remains unscheduled.

\noindent {\textbf{Test and Swap Assignment.}}
If the \algda has batches that \emph{violate their deadline}, i.e., scheduled in $S$ but not in $P$, we use the \algtas to improve the schedule.
Let $P^+_i =  \bigcup_i {P_{b^k_i}}$ be the union of the preferred intervals forming the \emph{total preferred interval} for an activity $\alpha_i$.
Each batch $b^k_i$ is tested for violating its deadline. \ysnoted{If $b^k_i$ violates, then $b^{k+1}_i$ must also violate}
If it violates, then batches $b^h_j$ from other activities already scheduled in $P^+_i$ are identified and tested if they too violate their deadline.
If so, $b^h_j$ is moved to the next available slot in $S_{b^h_j}$, and its old time slot given to $b^k_i$.
If $b^h_j$ is in its \emph{preferred interval} but has more slots available in this interval, then $b^h_j$ is moved to another free slot in $P_{b^h_j}$ and $b^k_i$ assigned to the slot that is freed. 
Else, the current configuration does not contain violations, except for the current batch $b^k_i$, but all available slots are occupied.
So, the utility for $b^k_i$ is compared with another $b^h_j$ in $P^+_i$,
and the batch with a higher utility gets this slot.


\subsubsection{The Core of \algjsc}
The \algjsc algorithm works as follows (Algorithm~\ref{alg:algjsc}).
After the initial \emph{clustering phase}, activities are tested for their feasibility.
If so, the \algda is initially evaluated in terms of total utility.
If this creates deadline violation, the \algtas performed, and the best scheduling is applied.


\begin{algorithm}[htbp]
    $\mathbb{C} \gets $ \emph{clustering phase}\;\label{code:2-preproc}
    \For {$C_k \in \mathbb{C}$} {
        \For {$\alpha_i \in C_k$} {
            \For {$d_j$ assigned to $C_k$} {
                \If {$\alpha_i \cup A_j$ is feasible} {\label{code:2-feasibility}
                    apply best scheduling among \emph{default} and \algtas on $\widehat{B_i}$\;\label{code:2-best}
                    
                    
                }
            }
        }
    }
	\caption{$\algjsc(A, D)$}
	\label{alg:algjsc}
\end{algorithm}

\subsubsection{Time Complexity of \algjsc}
ST-DBSCAN's time complexity is $\mathcal{O}(n \log n)$ for $n$ waypoints.
Unlike $k$-means clustering, ST-DBSCAN automatically picks a suitable number of clusters,
$k$, with $\approx \frac{n}{k}$ waypoints each.
For $k$ times, we compute the min-max of sets of size $\frac{n}{k}$, sort the $k$ elements and finally make $\frac{n}{k}$ assignments.
So this drones-to-clusters allocation takes $\mathcal{O}(k \frac{n}{k} + k \log k + \frac{n}{k})$ time.
Hence, this \emph{clustering phase} takes $\mathcal{O}(n \log n)$ time.


For the \algtas, we maintain an interval tree for fast temporal operations. 
If $l$ is the maximum number of batches to schedule per activity, building the tree costs $\mathcal{O}(\frac{nl}{k} \log(\frac{nl}{k}))$, while search, insertion and deletion cost $\mathcal{O}(\log(\frac{nl}{k}))$.
Finding free time slots makes a pass over the batches in $\mathcal{O}(\frac{nl}{k})$. 
This is repeated for $l$ batches, to give an overall time complexity of $\mathcal{O}(\frac{nl}{k} \log(\frac{nl}{k}) + \frac{n}{k}l^2)$.
Also the \algda relies on the same interval tree, reporting the same complexity as \algtas.

Finally, for the $k$ clusters and each application in a cluster, two schedule assignments are calculated for all the drones. 
Thus, the time complexity of \algjsc is $\mathcal{O}(n \log n) + \mathcal{O}(k \frac{n}{k}  m (\frac{nl}{k} \log(\frac{nl}{k}) + \frac{n}{k}l^2))$.
However, since the clustering can result in single cluster, $m \rightarrow n$, and the \emph{overall complexity} of \algjsc is $\mathcal{O}(n^3 l^2)$ in the worst case.

\subsection{The \algvrc Algorithm}
The \algvrclong (\algvrc) algorithm aims to find near-optimal waypoint routing while initially ignoring efficient scheduling of the batch computation.
\algvrc is split into three phases: \emph{routing}, \emph{splitting}, and \emph{scheduling}. 

\subsubsection{Routing Phase}
In this phase, \algvrc builds routes while satisfying the \emph{temporal constraint} for activities, i.e., for any two consecutive activities $(\alpha_i, \alpha_{i+1})$ in the route, $\bar{t}_i + \mathcal{F}(\lambda_i, \lambda_{i+1}) \le t_{i+1}$.
This is done using a modified version of $k$-nearest neighbors (\algknn) algorithm, whose solution is then locally optimized using the \algtwoopt heuristic~\cite{potvin1995exchange}.

The modified \algknn works as follows:
Starting from \depot, a route is iteratively built by selecting, from among the $k$ nearest waypoints which meet the \emph{temporal constraint}, 
the one, say, $\lambda_1$ whose activity has the earliest \emph{observation start time}. 
This process resumes from $\lambda_1$ to find $\lambda_2$, and so on until there is no feasible neighbor. \depot is finally added to conclude the route.
This procedure is repeated to find other routes until all the possible waypoints are chosen.
This initial set of routes is optimized to minimize the flying and hovering energy using \algtwoopt, which lets us find a local optimal solution from the given one~\cite{toth2002vehicle}.
However, routes found here may be infeasible for a drone to complete within its energy constraints.

\subsubsection{Splitting Phase}
Say $R_{i,j} = (\depot, \lambda_i, \ldots, \lambda_j, \depot)$ be an energy-infeasible route from the routing phase, which visits $\lambda_i$ and $\lambda_j$ as the first and last waypoints from \depot. 
The goal is to find a suitable waypoint $\lambda_g$ for $i \le g < j$ such that by splitting $R_{i,j}$ at $\lambda_g$ and $\lambda_{g+1}$, we can find an energy-feasible route while also improving the overall utility and reducing scheduling conflicts for batches.
For each edge $(\lambda_g,\lambda_{g+1})$, we compute a \emph{split score} whose value sums up three components: \emph{energy score}, \emph{utility score}, and \emph{compute score}.


\noindent {\textbf{Energy score.}}
Let $E(a,b)$ be the cumulative flying and hovering energy required for some route $R_{a,b} \subseteq R_{i,j}$. 
Here we sequentially partition the route $R_{i,j}$ into multiple \emph{viable trips} $R_{(i,k_{1}-1)}$, $R_{(k_{1},k_{2}-1)}$, \ldots, $R_{(k_{x},j)}$ such that each is a maximal trip and is energy-feasible, i.e., $E(k_{y},k_{y+1}-1) \leq E$ while $E(k_{y},k_{y+1}) > E$.
For each edge $(\lambda_g, \lambda_{g+1}) \in R_{(k_{y},k_{y+1}-1)}$, the \emph{energy score} is the ratio $\frac{E(k_{y},g)}{E} \leq 1$.
A high value indicates that a split at this edge improves the battery utilization.


\noindent \textbf{{Utility score.}}
Say $U(a,b)$ gives the cumulative data capture utility from visiting waypoints in a route $R_{a,b} \subseteq R_{i,j}$. Say edge $(\lambda_g, \lambda_{g+1}) \in R_{(k_{y},k_{y+1}-1)} \subseteq R_{i,j}$ is also part of a viable trip from above. Here, we find the data capture utility of a sub-route of $R_{i,j}$ that starts a new maximal viable trip at $\lambda_{g+1}$ and spans until $\lambda_{l}$, as $U(g,l)$. The utility score of edge $(\lambda_g, \lambda_{g+1})$ is the ratio between this new maximal viable trip and the original viable trip the edge was part of, $\frac{U(g,l)}{U(k_{y},k_{y+1}-1)}$. 
A value $>1$ indicates that a split at this edge improves the utility relative to the earlier sequential partitioning of the route.




\noindent  \textbf{{Compute score.}}
We first do a soft scheduling of the batches of all waypoints in $R_{i,j}$ using the \emph{first-fit} scheduling policy, mapping them to their \emph{preferred interval}, which is assumed to be free. Say there are $|R_{i,j}|$ such batches.
Then, for each edge edge $(\lambda_g, \lambda_{g+1}) \in R_{i,j}$, we find the overlap count $O_{g}$ as the number of batches from $\alpha_g$ whose execution slot overlaps with batches from all other activities. 
The overlap score for edge $(\lambda_g, \lambda_{g+1})$ is given as $\frac{O_{g}}{|R_{i,j}|}$.
If this value is higher, splitting the route at this point will avoid batches from having schedule conflicts in their preferred time slot.


Once the three scores are assigned, the edge with the highest \emph{split score} is selected as the split-point to divide the route into two sub-routes.
If a sub-route meets the energy constraint, it is selected as a \emph{valid trip}. 
If either or both of the sub-routes exceed the energy capacity, the splitting phase is recursively applied to that sub-route till all waypoints in the original route are part of some valid trip.

\subsubsection{Scheduling Phase}
Trips are then sorted in decreasing order of their total utility, and drones are allocated to trips depending their temporal availability.
Once assigned to a trip, the drone's scheduling is done by comparing the \algda and the \algtas used in \algjsc.


\subsubsection{The Core of \algvrc}
The \algvrc algorithm works as follows (Algorithm~\ref{alg:algvrc}).
After the initial \emph{routing phase}, energy-unfeasible routes are split into feasible ones in the \emph{splitting phase}, and then drones are allocated to them.
Finally, the \emph{scheduling phase} is applied to find the best schedule between the \algda and the \algtas.

\begin{algorithm}[htbp]
	$\mathbb{R} \gets $ \emph{routing phase}\;
	\For {$R_{ij} \in \mathbb{R}$} {
		\For {$(\lambda_g, \lambda_{g+1}) \in R_{ij}, i \le g < j$} {
			$s(g) \gets $ \emph{energy score} + \emph{utility score} + \emph{compute score}\;
		}
	}
	$\mathbb{R}' \gets $ \emph{splitting phase} based on scores $s(i), 1 \le i \le n$\;
	\For {$d_j$ assigned to $R_{ij} \in \mathbb{R}'$} {
		apply best scheduling among \algda and \algtas on $R_{ij}$\;
	}
	
	\caption{$\algvrc(A, D)$
	}
	\label{alg:algvrc}
\end{algorithm}

\subsubsection{Time Complexity of \algvrc}
In the \emph{routing phase}, the modified \algknn takes $\mathcal{O}(kn)$, with $n$ waypoints and $k$ number of neighbors. The \algtwoopt algorithm has time complexity $\mathcal{O}(n^4)$. Hence this phase overall has a cost of $\mathcal{O}(n^4)$.

In the \emph{splitting phase}, calculating the energy score for a route with length $n$ edges takes $\mathcal{O}(n)$. Calculating the energy score has $\mathcal{O}(n^2)$ complexity, and calculating the compute score has $\mathcal{O}(n)$ complexity. Considering a recursion of length $n-1$, the complexity of this phase is $\mathcal{O}(n^3)$
Combining \algda and \algtas, \algvrc's \emph{overall complexity} is $\mathcal{O}(n^4)$ in the worst case.


\section{Performance Evaluation}\label{sec:evaluation}
\subsection{Experimental Setup}

The \algopt solution is implemented using IBM's CPLEX MILP solver v12~\cite{cplex2009v12}. It uses Python to wrap the objective and constraints, and invokes the parallel solver. Our \algjsc and \algvrc heuristics have a sequential implementation using native Python. 
By default, these scheduling algorithms on our workloads run on an \emph{AWS c5n.4xlarge VM} with Intel Xeon Platinum 8124M CPU, $16$ cores, $3.0 \unit{GHz}$, and $42 \unit{GB}$ RAM.
\algopt runs on $16$ threads and the heuristics on $1$ thread.



We perform real-world benchmarks on flying, hovering, DNN computing and endurance, for a fleet of custom, commercial-grade drones. The X-wing quad-copter is designed with a top speed of $6 \unit{m/s}$ ($20 \unit{km/h}$), $<120 \unit{m}$ altitude, a $24000 \unit{mAh}$ Li-Ion battery and a payload capacity of $3 \unit{kg}$. It includes dual front and downward HD cameras, GPS and LiDAR Lite, and uses the Pixhawk2 flight controller. It also has an NVIDIA Jetson TX2 compute module with $4$-Core ARM64 CPU, $256$-core Pascal CUDA cores, $8 \unit{GB}$ RAM and $32 \unit{GB}$ eMMC storage. 
The maximum flying time is $\approx 30 \unit{min}$ with a range of $3.5 \unit{km}$.
Based on our benchmarks, we use the following drone parameters in our analytical experiments. 

\begin{center}
\small
\setlength\tabcolsep{2pt}
\begin{tabular}{c|c|c|c|c}
\toprule
 $s$ &  $\epsilon^f$ &  $\epsilon^h$ &  $\epsilon^c$ &  $E$ \\
\midrule
$4 \unit{m/s}$ & $750 \unit{J/s}$ & $700 \unit{J/s}$ & $20 \unit{J/s}$ & $1350 \unit{kJ}$\\
\bottomrule
\end{tabular}
\end{center}


\begin{figure*}[t]
\centering
  \subfloat[Utility per drone, RND]{
    \includegraphics[width=0.42\textwidth]{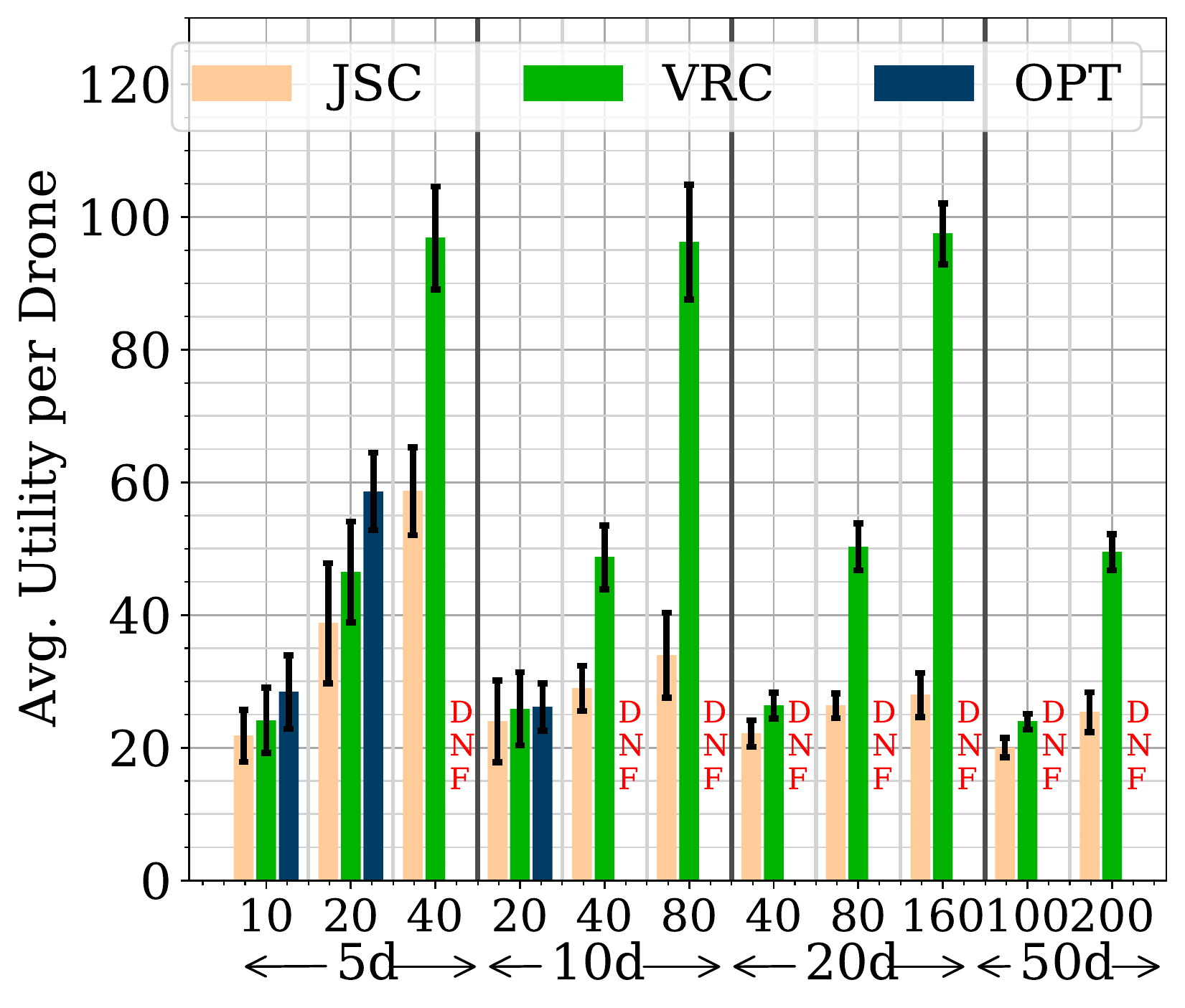}
    \label{fig:exp:w1:util}
  }
  \subfloat[Utility per drone, DFS]{
    \includegraphics[width=0.42\textwidth]{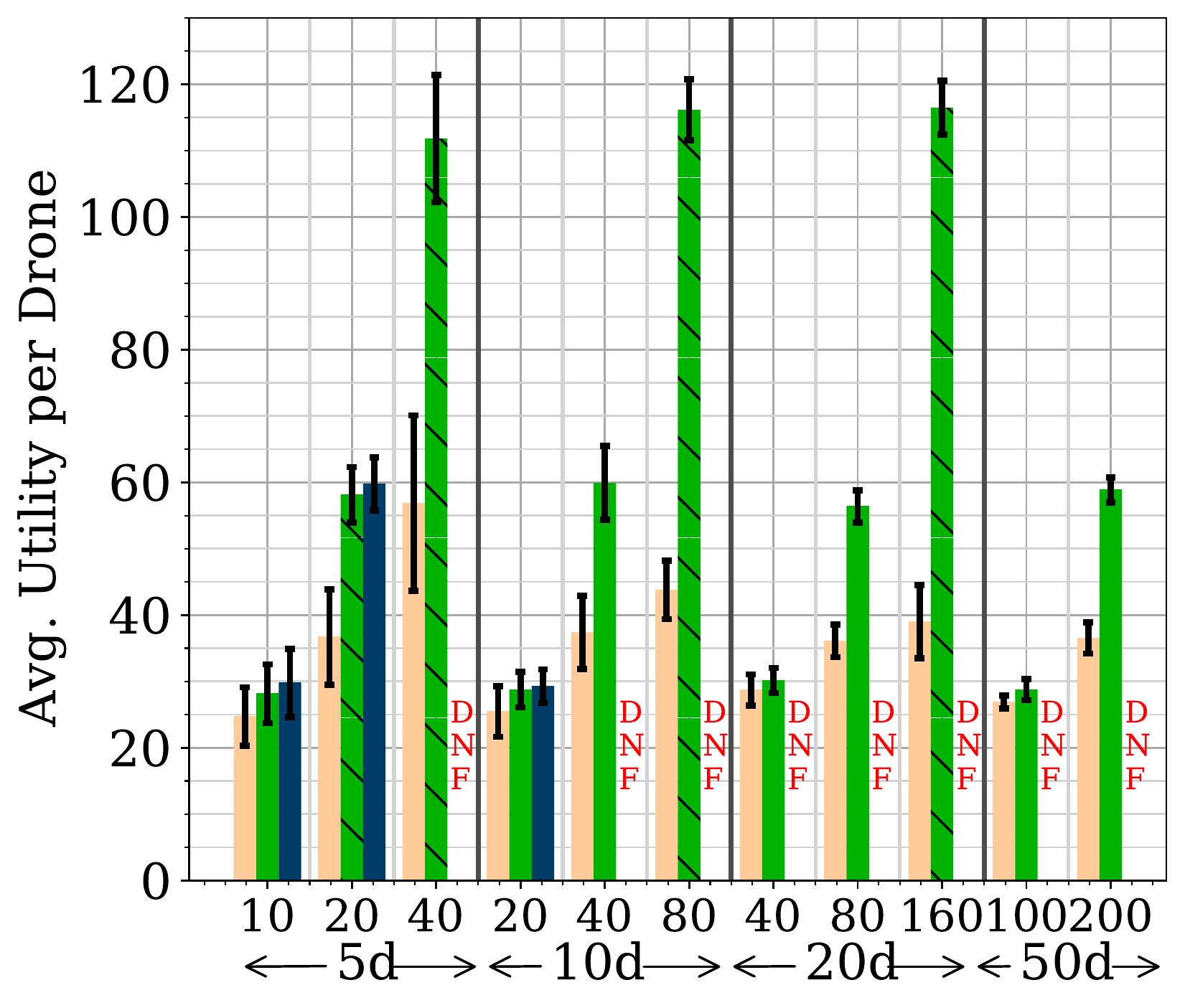}
    \label{fig:exp:w2:util}
  }
\hfill
  \subfloat[Alg. runtime, RND]{
    \includegraphics[width=0.42\textwidth]{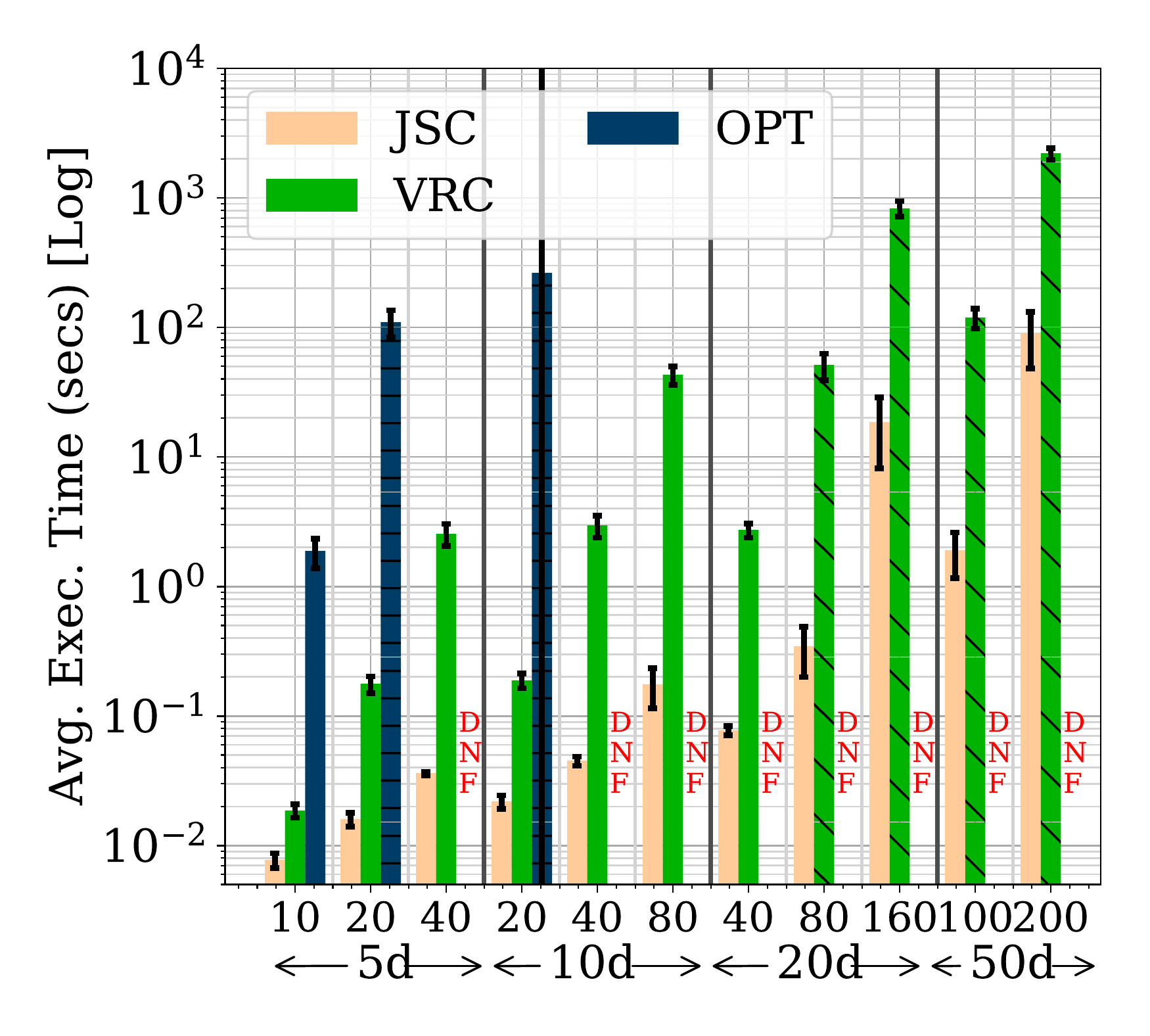}
    \label{fig:exp:w1:exec}
  }
  \subfloat[Alg. runtime, DFS]{
    \includegraphics[width=0.42\textwidth]{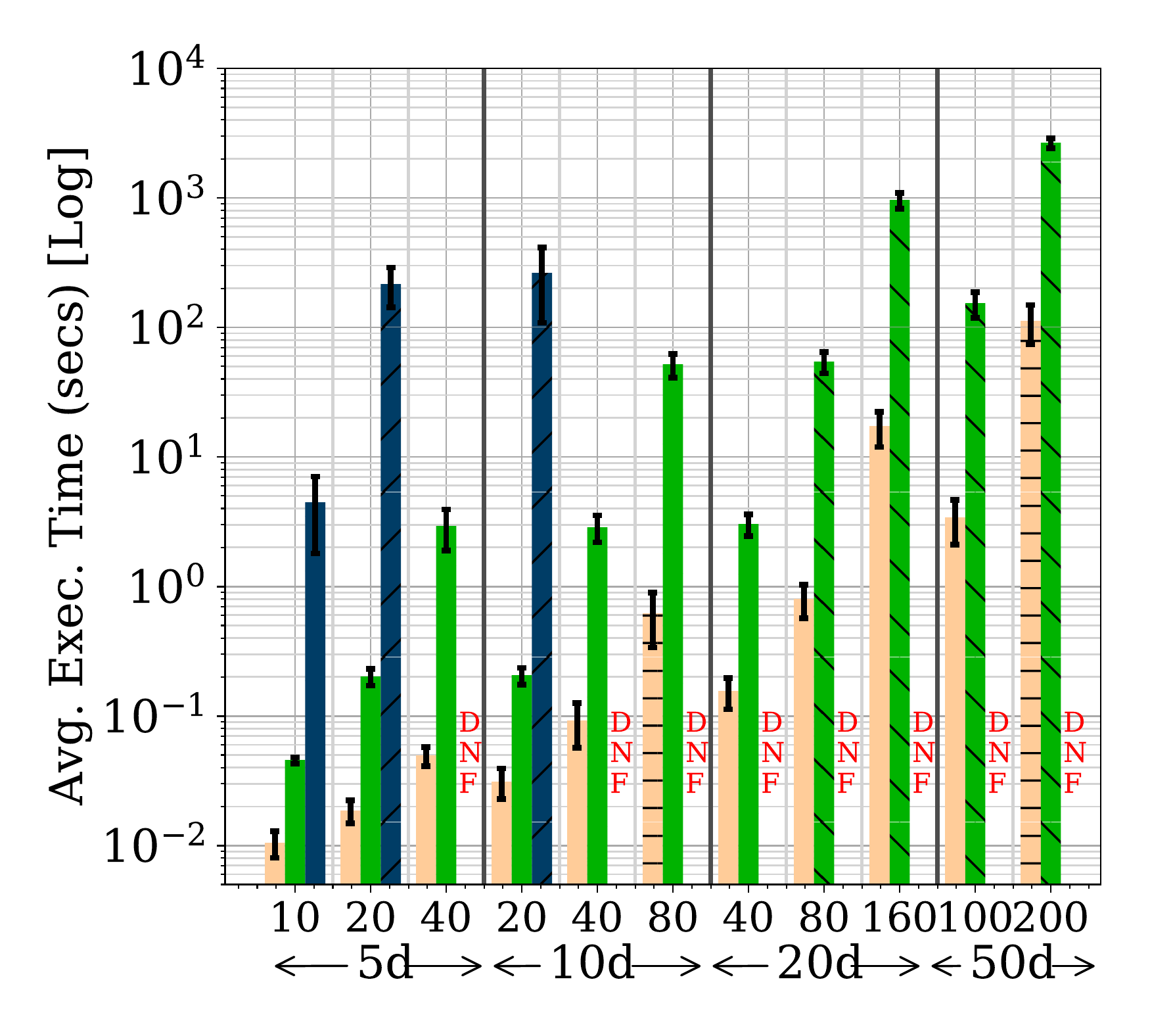}
    \label{fig:exp:w2:exec}
  }
\caption{\emph{Expected utility per drone} and \emph{algorithm runtime} of the three MSP algorithms, for the RND and DFS workloads on MNet. On the X axis, the number of drones (outer) and activities per drone (inner) increase. \algopt is solved on $16\times$ cores while \algjsc and \algvrc run on just $1$. DNF indicates \algopt did not finish.}
\label{fig:exp:util-exec}
\end{figure*}

\subsection{Workloads}

We evaluate the scheduling algorithms for two \emph{application workloads}: Random (RND) and Depth First Search (DFS). Both have a maximum mission time of $4 \unit{h}$ over multiple trips. In the \emph{RND workload}, $n$ waypoints are randomly placed within a $3.5 \unit{km}$ radius from the depot, and with a random activity start time within $(0,240] \unit{mins}$. This is an adversarial scenario with no spatio-temporal locality. The \emph{DFS workload} is motivated by realistic traffic monitoring needs. We perform a depth-first traversal over a $3.5 \unit{km}$ radius of our local city's road network, centered at the depot. With a $\mathcal{P}=\frac{1}{10}$ probability, we pick a visited vertex as an activity waypoint; $\mathcal{P}$ grows by $\frac{1}{10}$ for every vertex that is not selected, and $n$ are chosen. The start time of these activities monotonically grows.

The table below shows the activity and drone \emph{scenarios} for each workload. These are based on reasonable operational assumptions and schedule feasibility. We vary the data capture time ($\bar{t}-t$); batching interval ($\beta$); batch execution time on $2$ DNNs ($\rho_M$, $\rho_R$)\footnote{We run \emph{SSD Mobilenet v2 DNN} (MNet, $\rho_M$)~\cite{Sandler_2018_CVPR}, popular for analyzing drone footage~\cite{wang2018bandwidth}, and \emph{FCN Resnet18 DNN} (RNet, $\rho_R$)~\cite{long2015fully} on the TX2.}; deadline ($\delta$); utility ($\gamma$); and number of drones ($m$). The \emph{load factor} $x$ decides the count of activities per mission, $n = x \cdot m$. Drones take at most  $r_{\max}=\frac{n}{m}$ trips.
\begin{center}
\small
\setlength\tabcolsep{2pt}
\begin{tabular}{c|c|c|c|c|c|c|c|c}
\toprule 
 $\bar{t}-t$ &  $\beta$ & $\rho_M$ & $\rho_R$ & $\delta$ & $\gamma$ & $m$ & $x$ & $n= x \cdot m$ \\
\midrule
$[1,5]$ & $60\unit{s}$ & $11\unit{s}$ & $98\unit{s}$ &$120\unit{s}$ & $[1,5]$ & $5,10,20,50$ & $2, 4, 8$ & $10,\ldots,200$\\
\bottomrule
\end{tabular}
\end{center}



\noindent For brevity, RNet is only run on DFS. $10$ \emph{instances} of each of these $33$ viable workload scenarios are created. We run \algopt, \algjsc and \algvrc for each to return a schedule and expected utility.

\subsection{Experimental Results}


Figures~\ref{fig:exp:w1:util},~\ref{fig:exp:w2:util} and~\ref{fig:exp:w2r:util} show the \emph{expected utility per drone} for the schedules from the $3$ algorithms, for different drone counts and activity load factors. Similarly, Figures~\ref{fig:exp:w1:exec},~\ref{fig:exp:w2:exec}, and~\ref{fig:exp:w2r:exec} show the \emph{algorithm execution time} (log, $secs$) for them. Each bar is averaged for $10$ instances and the standard deviations shown as whiskers. The per drone utility lets us uniformly compare the schedules for different workload scenarios. The \emph{total utility} -- MSP objective function -- is the product of the per drone utility shown and the drone count. \algopt \emph{did not finish} (DNF) within $7 \unit{h}$ for scenarios with $40$ or more activities.

\begin{figure}[t]
\centering
  \subfloat[Utility per drone]{
    \includegraphics[width=0.42\textwidth]{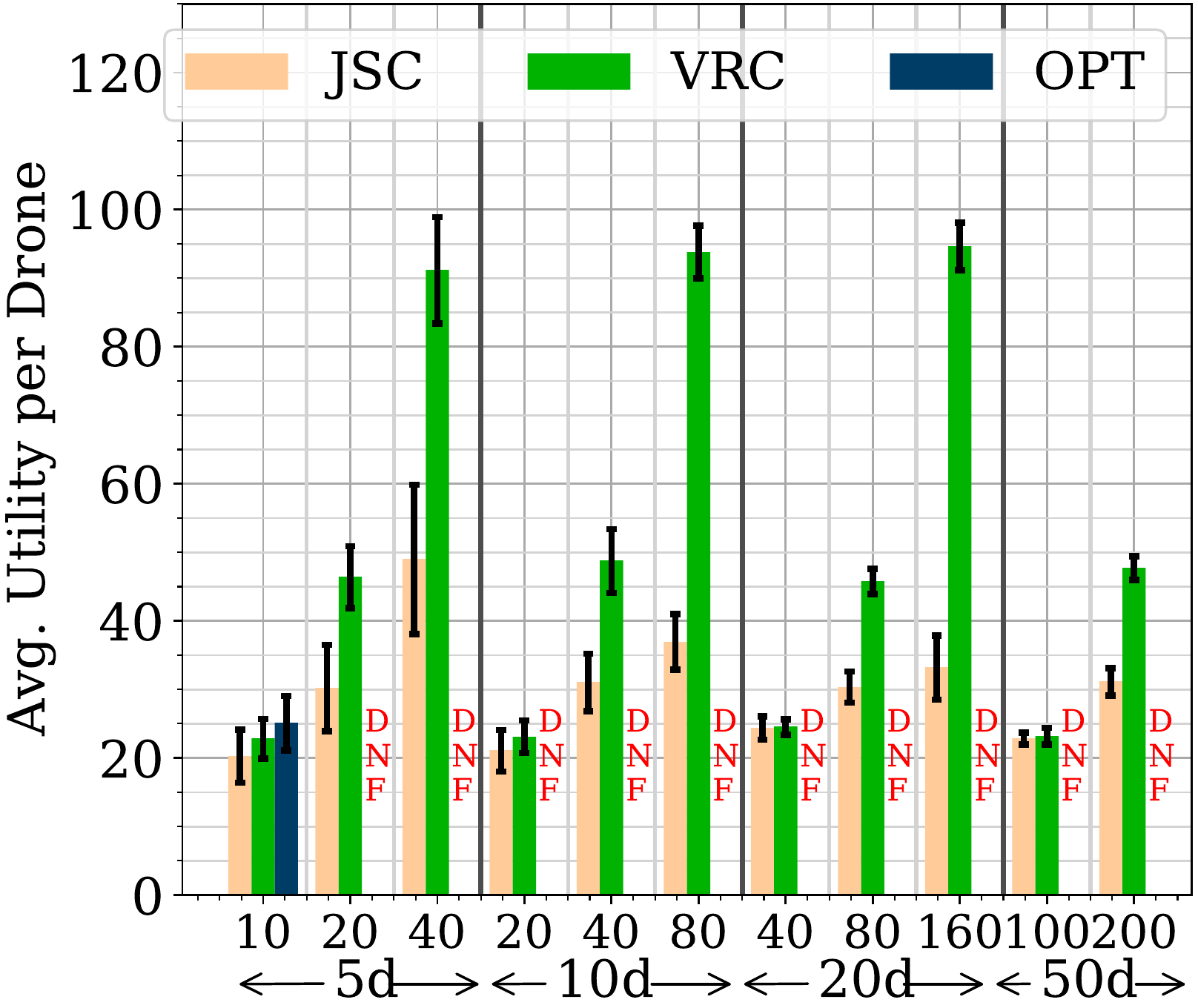}
    \label{fig:exp:w2r:util}
  }%
\hfill
  \subfloat[Alg. runtime]{
    \includegraphics[width=0.42\textwidth]{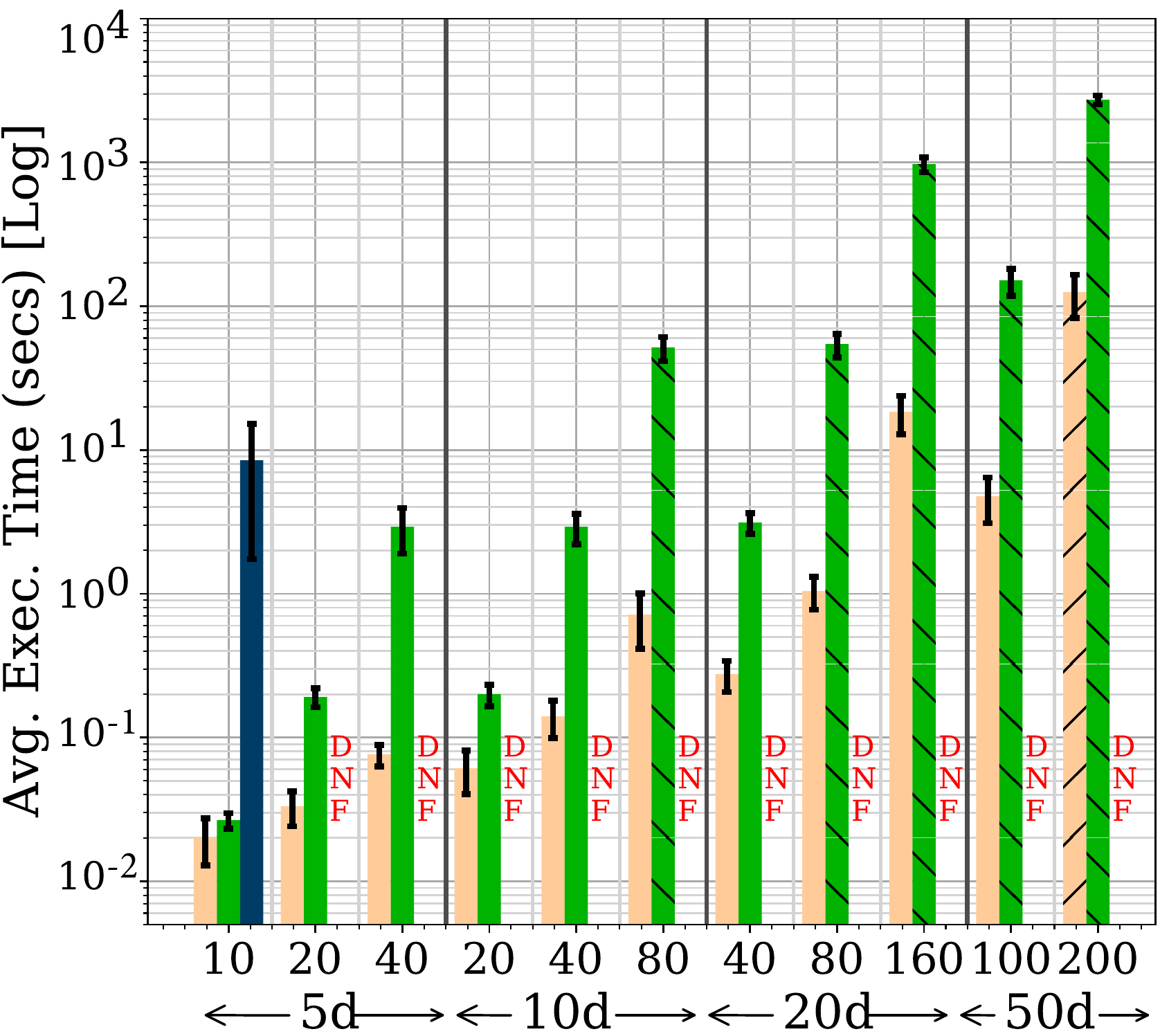}
    \label{fig:exp:w2r:exec}
  }
\caption{\emph{Expected utility per drone} and \emph{algorithm runtime} for RNet on DFS.}
\label{fig:exp:util-exec:w2r}
\end{figure}

\subsubsection{\algopt offers the highest utility, if it completes executing, followed by \algvrc, and \algjsc}
Specifically, for the $5$-drone scenarios for which \algopt completes, it offers an average of $42\%$ higher expected utility than \algjsc. \algvrc gives $26\%$ more average utility than \algjsc for these scenarios, and $75\%$ more for all scenarios they run for.
%
This is as expected for \algopt. Since a bulk of the energy is consumed for flying and hovering, \algvrc, which starts with an energy-efficient route, schedules more activities within the time and energy budget, as compared to \algjsc. 
This is evidenced by Figure~\ref{fig:exp:act}, which reports for MNet the \emph{average fraction of activities}, which are submitted and successfully scheduled by the algorithms. The remaining activities are not part of any trip. Among all workloads, \algjsc only schedules $60\%$ of activities, VRC $90\%$, and \algopt $98\%$. So \algopt and \algvrc are better at packing routes and analytics on the UAVs.
\algopt and \algvrc offer more utility for the DFS workload than RND
since $\geq 96\%$ of DFS activities are scheduled.
They exploit the spatial and temporal locality of activities in DFS.

\subsubsection{The average flying time per activity in each trip is higher for \algvrc compared to \algjsc} Interestingly, at $728 \unit{s}$ vs. $688 \unit{s}$ per activity, the route-efficient schedules from \algvrc manage to fly to waypoints farther away from the depot and/or from each other, within the energy constraints, when compared to the schedules from \algjsc. 
\ysnoted{??? Avg dist of waypoints visited by VRC in a trip vs JSC in a trip; Former may be more.}
As a result, it schedules a larger fraction of the activities to gain a higher expected utility.

\ysnoted{comments on std dev}

\begin{figure}[tbp]
\centering
  \subfloat[RND]{
    \includegraphics[width=0.42\textwidth]{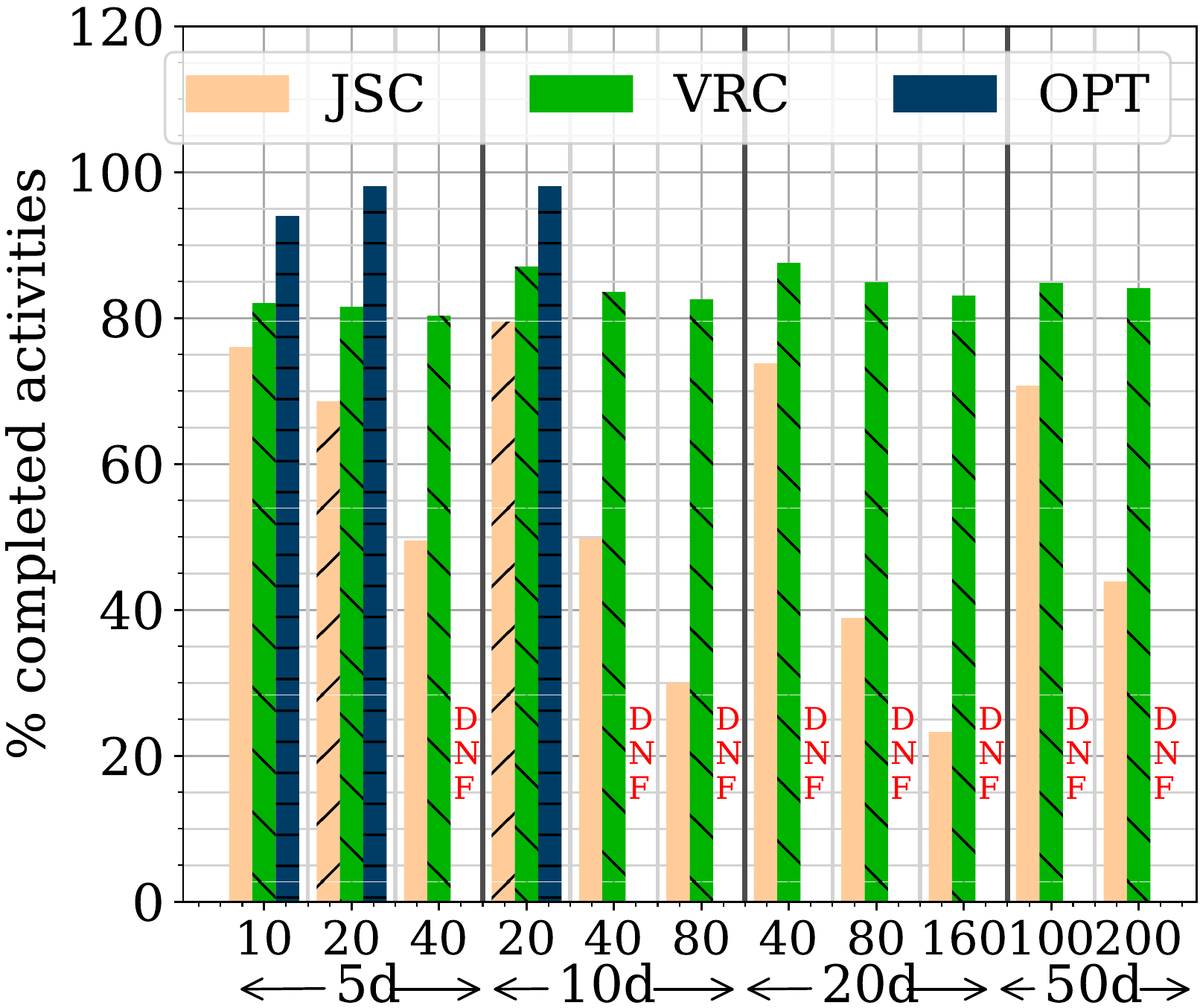}
    \label{fig:exp:w1:act}
  }
	\hfill
  \subfloat[DFS]{
    \includegraphics[width=0.42\textwidth]{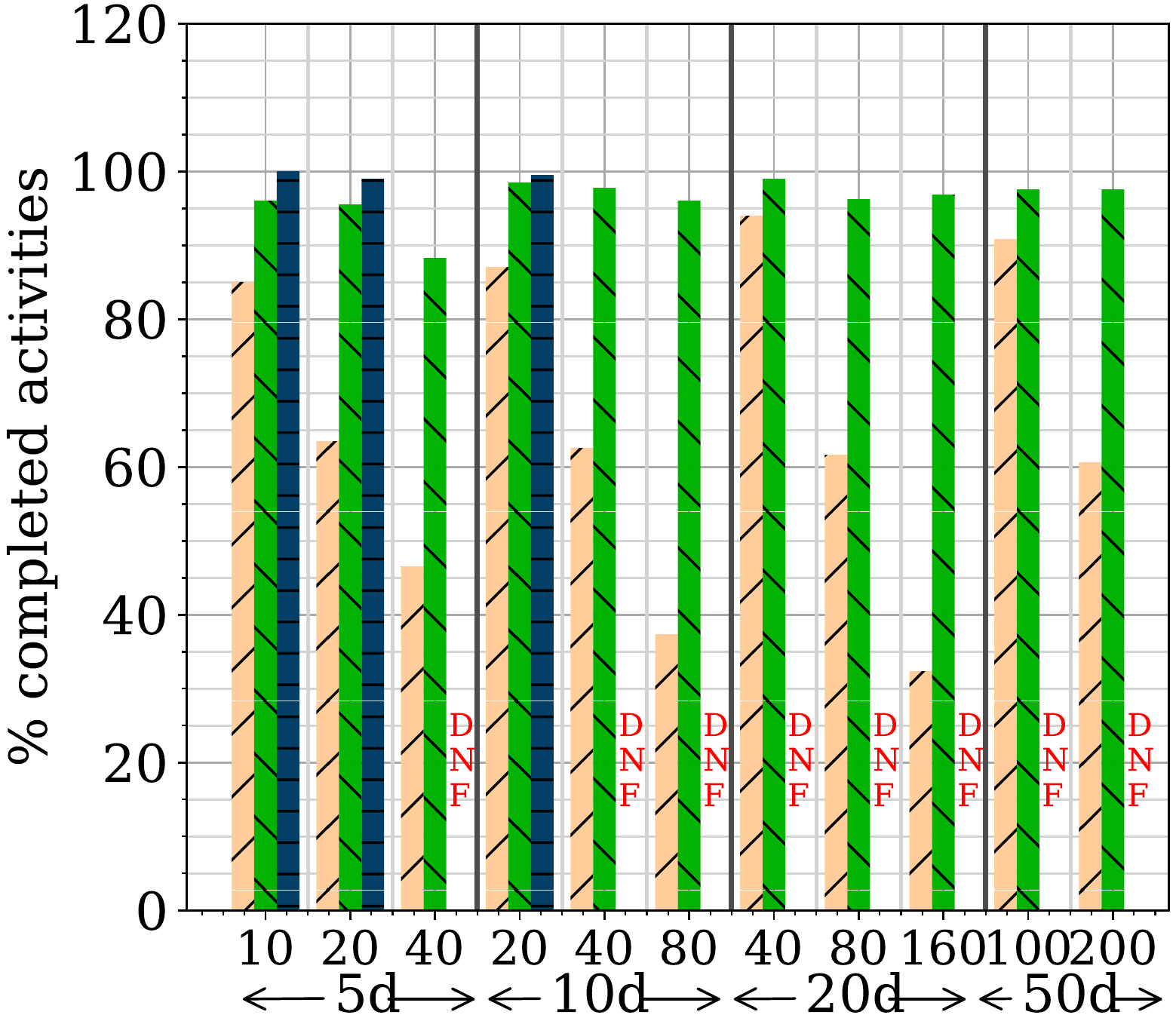}
    \label{fig:exp:w2:act}
}
\caption{Fraction (\%) of submitted activities scheduled per mission for MNet.}
\label{fig:exp:act}
\end{figure}

\subsubsection{The execution times for \algvrc and \algjsc match their time complexity}
We use the execution times for \algjsc to schedule the $300+$ workload instances to fit a \emph{cubic function} in $n$, the number of activities, to match its time complexity of $\mathcal{O}(n^3 \cdot l^2)$; since in our runs, $l \in [1,5]$ and $l \leq n$, we omit that term in the fit.
Similarly, we fit a \emph{degree-4 polynomial} for \algvrc in $n$.
The \emph{correlation coefficient} for these two fits are high at $0.86$ and $0.99$, respectively. So, the real-world execution time of our scheduling heuristics match our complexity analysis.


\subsubsection{\algopt is the slowest to execute, followed by \algvrc and \algjsc}
Despite \algopt using $16\times$ more cores than \algjsc and \algvrc, its average execution times are $>100 \unit{s}$ for just $20$ activities. The largest scenario to complete in reasonable time is $40$ activities on $5$ drones, which took $7 \unit{h}$ on average. This is consistent with the NP-hard nature of MSP. As our mission window is $4 \unit{h}$, any algorithm slower than that is not useful. 

\algjsc is fast, and on average completes within $1 \unit{s}$ for up to $80$ activities. Even for the largest scenario with $50$ drones and $200$ activities, it takes only $90 \unit{s}$ for RND and $112 \unit{s}$ for DFS.
\algvrc is slower but feasible for a larger range of activities than \algopt. It completes within $3 \unit{min}$ for up to $100$ activities. But, it takes $\approx 45 \unit{min}$ to schedule $200$ activities on $50$ drones. 
%

\subsubsection{The choice of a good scheduling algorithm depends on the fleet size and activity count}
From these results, we can conclude that \algopt is well suited for small drone fleets with about $20$ activities scheduled per mission. This completes within minutes and offers about $20\%$ better utility than \algvrc.
\algvrc offers a good trade-off between utility and execution time for medium workloads with $100$ activities and $50$ drones. This too completes within minutes and gives on average about $75\%$ better utility than \algjsc and schedules over $80\%$ of all submitted activities.
For large fleets with $200$ or more activities being scheduled, \algjsc is well suited with fast solutions but has low utility and leaves a majority of activities unscheduled.





\subsubsection{A higher load factor increases the utility, but causes fewer \% of activities to be scheduled} 
As $x$ increases, we see that the utility derived increases. This is partly due to adequate energy and time being available for the drones to complete more activities in multiple trips.
E.g., for the 5-drone case, we use load factors of $x=\{2,4,8,16,32\}$ for \algjsc and \algvrc. There is a consistent growth in the total utility, from $109$ to $523$ for \algjsc, and from $121$ to $1080$ for \algvrc. There is also a corresponding growth in the number of trips performed per mission, e.g., from $7.5$ to $43.2$ in total for \algvrc.

However, the fraction of submitted activities that are scheduled falls. For \algjsc, its activity scheduled \% linearly drops with $x$ from $76\%$ to $23\%$. But for \algvrc, the scheduled \% stays at about $80\%$ until $x=8$, at which point the activities saturate the drone fleet's capacity and the scheduled \% falls linearly to $37\%$ for $x=32$.
Interestingly, the utility increases faster than the number of activities scheduled for \algvrc. This is due to the scheduler favoring activities that offer a higher utility, while avoiding those with a lower utility, causing a $20\%$ increase in utility received per activity between $x=8$ to $x=32$. 

\subsubsection{Longer-running edge analytics offer lower on-time utility} 
\ysnoted{Why does DNN2 not compelte on OPT for 5/10s and beyond, within time?}
We run the same scenarios using RNet and MNet DNNs for the DFS workload. For both \algjsc and \algvrc, the \emph{data capture utility} that accrues from their schedules for the two DNNs is similar. However, since the RNet execution time per batch is much higher than MNet, there is a drop in \emph{on-time utility}, by about $32\%$ for both \algjsc and \algvrc, due to more deadline violations.
As a result, this also causes a drop in total utility for RNet by about $15.9\%$ for \algjsc and $19\%$ for \algvrc, relative to MNet. Even for \algopt we see a similar trend with a $15.8\%$ drop in the total utility. The runtime of \algjsc and \algvrc do not exhibit a significant change between RNet and MNet.

\subsubsection{Effect of real-world factors}
The \textit{expected utilities} reported above are under ideal conditions. Here, we evaluate their practical efficacy by emulating these schedules using real drone traces to get the \textit{effective utility} and \textit{trip completion rate}. 

Ideally, each trip generated by \algjsc and \algvrc should complete be within a drone's energy capacity. 
In practice, factors such as wind or non-linear battery performance can increase or decrease the actual energy consumed. Figure~\ref{fig:exp:real-trips} shows the \% of scheduled trips that do not complete when using the drone trace. With $<80$ activities, all trips complete (not plotted). But but with $\geq 80$ activities, some trips in the planned schedule start to fail. At worst, $12\%$ of trips are incomplete in some schedules. So the effect of real-world factors can be significant. Interestingly, for the failed trips, an average $3.6\%$ and a maximum of $7.9\%$ extra battery capacity would allow them to finish the trip.
So by maintaining a buffer battery capacity of $\approx 10\%$ when planning a schedule, we can ensure that the drones can complete a trip and return to the depot.

\begin{figure*}[t]
\centering
  \subfloat[MNet, RND]{
    \includegraphics[width=0.31\textwidth]{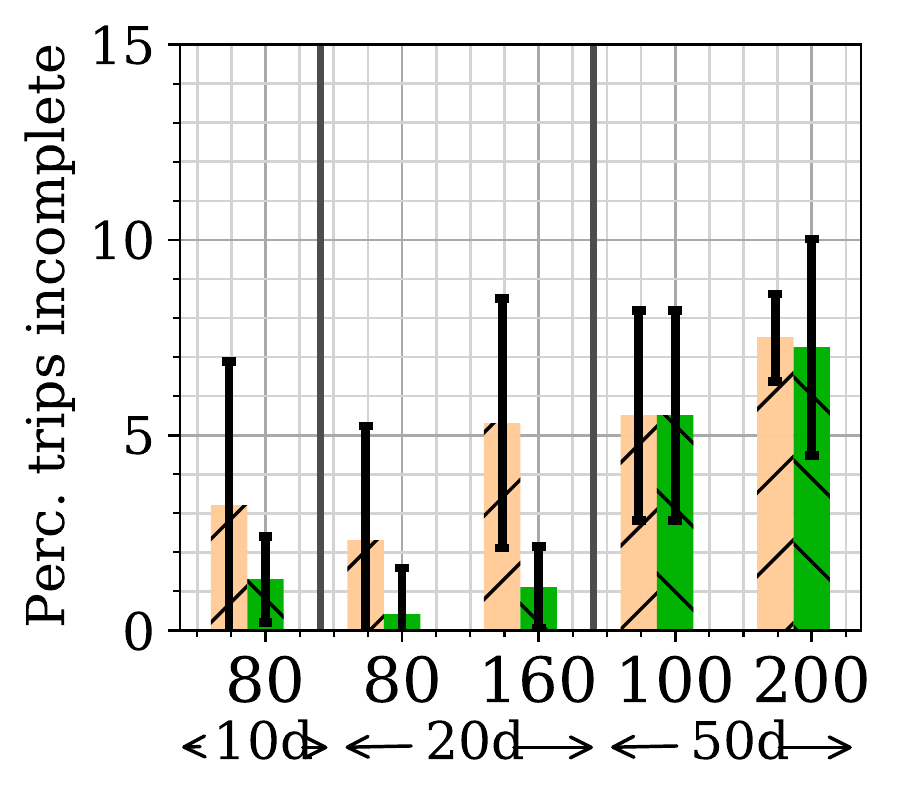}
    \label{fig:exp:w1:util_2}
  }
  \subfloat[MNet, DFS]{
    \includegraphics[width=0.31\textwidth]{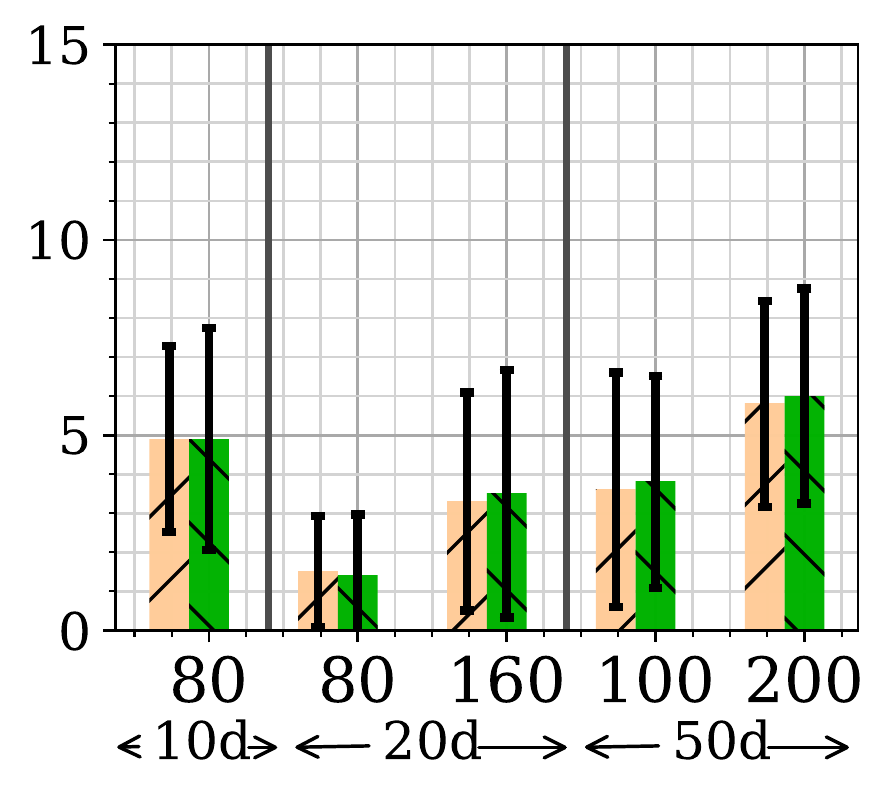}
    \label{fig:exp:w2:util_2}
  }
  \subfloat[RNet, DFS]{
    \includegraphics[width=0.31\textwidth]{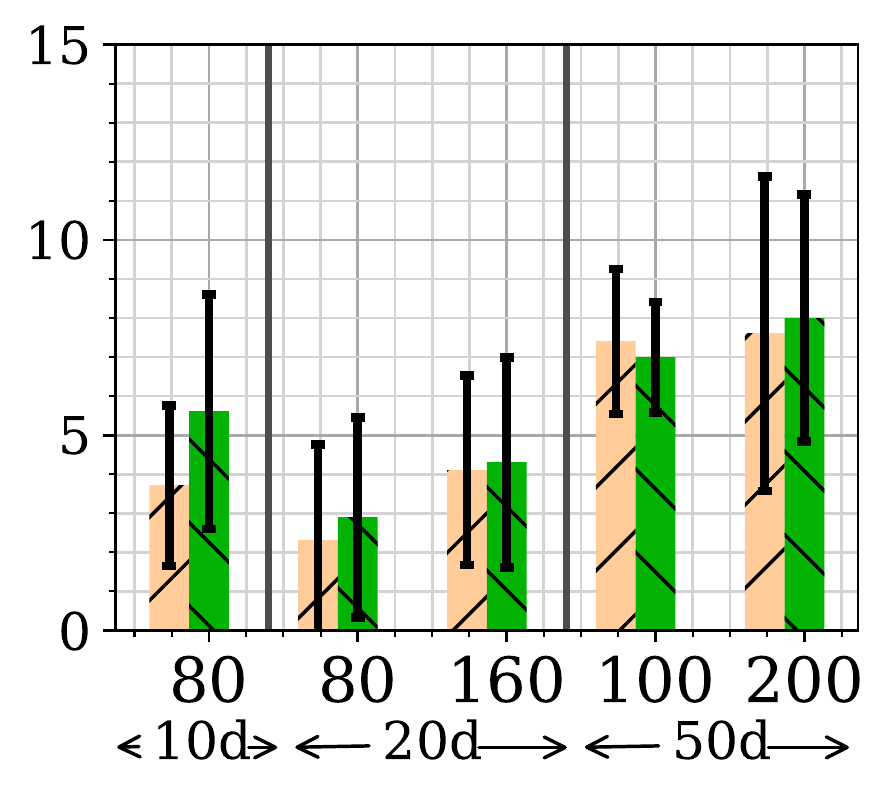}
    \label{fig:exp:w1:exec_2}
  }
\caption{\% of incomplete trips using drone trace. It is $0\%$ for $<80$ activities.}
\label{fig:exp:real-trips}
\end{figure*}

\section{Conclusion and Future Work}
\label{sec:conclusions}
\ysnoted{Add 1-2 paras with high level takeways from the current problem and its results...}

\ysnoted{parallel variants of VRC to speed up processing?}

This paper introduces a novel Mission Scheduling Problem (\prob) that co-schedules routes and analytics for drones, maximizing the utility for completing activities.
We proposed an optimal algorithm, \algopt, and two time-efficient heuristics, \algjsc and \algvrc.
Evaluations using two workloads, varying drone counts and load factors, and real traces exhibit different trade-offs between utility and execution time. \algopt is best for $\leq20$ activities and $\leq 5$ drones, \algvrc for $\leq100$ activities and $\leq 50$ drones, and \algjsc for $>100$ activities. 
Their time complexity matches reality. 
The schedules work well for fast and slow DNNs, though on-time utility drops for the latter.

The MSP proposed here is just one variant of an entire class of fleet co-scheduling problems for drones. Other architectures can be explored considering 4G/5G network coverage to send edge results to the back-end, or even off-load captured data to the cloud if it is infeasible to compute on the drone. This will allow more pathways for data sharing among UAVs and GS, but impose energy, bandwidth and latency costs for communications. Even the routing can be aware of cellular coverage to ensure deterministic off-loading on a trip.

We can use alternate cost models by assigning an operational cost per trip or per visit, and convert the MSP into a profit maximization problem. The activity time-windows may be relaxed rather than be defined as a static window. Drones with heterogeneous capabilities, in their endurance, compute capabilities, and sensors, will also be relevant for performing diverse activities such as picking up a package using an on-board claw and visually verifying it using a DNN.

Finally, we need to deal with dynamics and uncertainties like wind, obstacles and non-linear battery or compute behavior that affect flight paths, energy consumption and utilities. We can use probability distributions and stochastic approaches coupled with real-time information, which can decide and enact on-line rescheduling and rerouting while on a trip. 
Such on-the-fly route updates for drones also allows us to accept and schedule activities continuously, rather accumulate a mission over hours, and prioritize the profitable activities. These will also need to be validated using more robust real-world experiments and traces.
\let\thefootnote\relax\footnotetext{\noindent \textbf{Acknowledgments.} \textit{This work is supported by AWS Research Grant, Intelligent Systems Center at Missouri S\&T, and NSF grants CCF-1725755 and SCC-1952045. A. Khochare is funded by a Ph.D. fellowship from RBCCPS, IISc, Bangalore.
S. K. Das was partially supported by a Satish Dhawan Visiting Chair Professorship at IISc. We thank RBCCPS for access to the drone, and Vishal, Varun and Srikrishna for helping collect the drone traces.}}








\clearpage

\bibliographystyle{ieeetr}
\bibliography{main-infocom-21}

\end{document}